%% file: main.tex
\pdfoutput=1
\documentclass[ejsv2,preprint,noshowframe]{imsart}

\RequirePackage{amsmath,mathtools}
\makeatletter\@ifpackageloaded{newtxmath}{}{\RequirePackage{amssymb}}\makeatother
\RequirePackage{booktabs,graphicx,enumitem}
\RequirePackage{microtype}
\RequirePackage[numbers,sort&compress]{natbib}
\RequirePackage[colorlinks,citecolor=blue,urlcolor=blue]{hyperref}

\startlocaldefs

\theoremstyle{plain}
\newtheorem{theorem}{Theorem}
\newtheorem{lemma}[theorem]{Lemma}
\newtheorem{corollary}[theorem]{Corollary}
\newtheorem{proposition}[theorem]{Proposition}
\theoremstyle{definition}
\newtheorem{remark}[theorem]{Remark}
\newtheorem{example}[theorem]{Example}

\newcommand{\PP}{\mathbb{P}}
\newcommand{\EE}{\mathbb{E}}
\newcommand{\R}{\mathbb{R}}
\newcommand{\ind}[1]{\mathbf{1}\{#1\}}
\newcommand{\Fhat}{\widehat{F}_{Q,m}}
\newcommand{\what}{\widehat{w}}
\newcommand{\tauhat}{\widehat{\tau}}
\newcommand{\abar}{\bar{\alpha}}
\newcommand{\TV}{\mathrm{TV}}
\newcommand{\KL}{\mathrm{KL}}
\newcommand{\cX}{\mathcal X}
\newcommand{\cY}{\mathcal Y}
\newcommand{\cI}{\mathcal I}
\DeclareMathOperator{\Var}{Var}
\DeclareMathOperator{\sd}{sd}
\newcounter{procedure}
\endlocaldefs

\begin{document}
\begin{frontmatter}
\title{Sharp training-conditional coverage for conformal prediction under covariate shift}
\runtitle{Training-conditional coverage under covariate shift}

\begin{aug}
\author[A]{\fnms{Mehrdad}~\snm{Pournaderi}\ead[label=e1]{m.pournaderi@emofid.com}}
\address[A]{Mofid Securities\printead[presep={,\ }]{e1}}
\runauthor{M. Pournaderi}
\end{aug}

\begin{abstract}
Weighted split conformal prediction reweights calibration scores by the likelihood ratio between the test and training covariate distributions and guarantees marginal coverage under covariate shift. We study its coverage conditional on the calibration data. An elementary argument, based on a single concentration inequality at a fixed population quantile, gives explicit training-conditional bounds without unspecified constants, and shows that the relevant scale is not the supremum of the likelihood ratio but a variance proxy built from the chi-squared divergence of the shift and from the average of the ratio over the part of the test population, of probability equal to the miscoverage level, where it is largest. A two-point lower bound shows that the root-$m$ rate and the chi-squared contribution are intrinsic to the shift. Run at an explicitly inflated level, the weighted quantile becomes a deterministic PAC prediction set. We compare it with randomized rejection sampling and with importance-weighted learn-then-test and, through a certified choice of a clipping level for the likelihood ratio, map the regime in which each gives the narrower valid set. The analysis extends to estimated likelihood ratios and to tail functionals estimated from an unlabeled source sample, which yields a fully finite-sample certificate.
\end{abstract}



\end{frontmatter}

\section{Introduction}

A practitioner who uses conformal prediction calibrates once and deploys the result many times. A predictor $\widehat\mu$ is fitted on a training split, nonconformity scores $S_i=s(X_i,Y_i)$ are computed on a calibration split of size $m$, and every future prediction set is $\{y:s(x,y)\le\tauhat\}$ with $\tauhat$ an empirical quantile of the calibration scores \citep{vovk2005algorithmic,angelopoulos2023conformal}. When calibration and test points are exchangeable, $\PP(Y_{\rm test}\in\widehat C(X_{\rm test}))\ge1-\alpha$. This is a \emph{marginal} guarantee: it averages over calibration samples the practitioner might have drawn. What the practitioner experiences is the miscoverage of the one set they obtained,
\[
P_e(D):=\PP\big(Y_{\rm test}\notin\widehat C(X_{\rm test})\mid D\big),
\]
a random variable whose value depends on the calibration data $D$; in the exchangeable case it is centred at $\alpha$, so that, for continuous scores and large $m$, the deployed set under-covers for about half of the calibration samples. A \emph{training-conditional}, or PAC, guarantee bounds this random variable for most calibration samples: $P_e(D)\le\alpha+\gamma$ with probability $1-e^{-2m\gamma^2}$ over $D$ \citep{vovk2012conditional,bian2023training,duchi2025sample}, and since the bound is explicit, running the quantile at the level $1-\alpha+\gamma$ certifies $P_e(D)\le\alpha$ for a prescribed fraction $1-\delta$ of calibration samples.

Under \emph{covariate shift} \citep{shimodaira2000improving,sugiyama2007covariate}, the test covariate distribution $Q_X$ differs from the training distribution $P_X$ while $P_{Y|X}$ is unchanged. \citet{tibshirani2019conformal} showed that marginal coverage is restored by weighting the calibration scores with the likelihood ratio $v=dQ_X/dP_X$, and this weighted conformal prediction has become the standard tool for the problem \citep{lei2021conformal,prinster2022jaws,fannjiang2022conformal}. Its training-conditional behaviour is in general worse than in the exchangeable case, because the weighting can increase the variance of the empirical quantile. In previous work \citep{pournaderi2026training} we quantified this through a weighted Dvoretzky--Kiefer--Wolfowitz (DKW) inequality, obtaining, for $v\le B$,
\begin{equation}\label{eq:old}
\PP\Big(P_e(D)>\alpha+B\sqrt{2\log(4/\delta)/m}+2C\sqrt{B/m}\Big)\le\delta,
\end{equation}
where $C$ is the universal constant in a bracketing-entropy bound for empirical processes. The constant is unknown, so \eqref{eq:old} describes the rate but cannot be used to certify a deployed set.

\paragraph{A running example.} Consider a subgroup $A$ of the population that makes up $0.05\%$ of the training data but $1\%$ of the test population, and suppose the two populations agree otherwise: $Q_X=0.99\,P_X+0.01\,P_X(\cdot\mid A)$. The likelihood ratio is $v\approx21$ on $A$ and $v\approx0.99$ elsewhere, so $B=21$, and any bound scaled by $B$ treats the problem as if the calibration set were $21$ times smaller. Yet only $1\%$ of the test population is affected, which is less than the miscoverage level $\alpha=5\%$ one is trying to certify. The functionals introduced in Section~\ref{sec:upper} make this precise: the $\chi^2$-divergence of the shift is $0.2$, the average of $v$ over the $5\%$ of the test population where it is largest is $5.2$, and the variance proxy that governs the weighted quantile is four times smaller than the worst-case value $\alpha(1-\alpha)B$. We return to this example when clipping the ratio (Section~\ref{sec:clip}) and when its functionals have to be estimated (Section~\ref{sec:est}); it is the design of Sections~\ref{sec:exp_frontier}--\ref{sec:exp_est}.

\paragraph{Existing PAC constructions.} \citet{park2022pac} give two. The first is a conservative deterministic baseline: since
\[
Q(S>\tau)=\EE_P\big[v\ind{S>\tau}\big]\le B\,P(S>\tau),
\]
a Clopper--Pearson calibration of the \emph{source} scores at level $\alpha/B$ is PAC for the target, at the price of a set whose target miscoverage is typically far below $\alpha$. The second, their main proposal, accepts each calibration point with probability $v_i/B$, obtaining an exchangeable subsample of expected size $m/B$ from the test distribution, and applies a Clopper--Pearson bound to it. This is a randomized procedure: two analysts holding the same calibration data obtain different sets, with a variability that Section~\ref{sec:exp_frontier} measures; and it is tied to the supremum, since it keeps $m/B$ points---$5\%$ of the calibration set in the running example. Two deterministic constructions have appeared since. The importance-weighted learn-then-test procedure of \citet{almeida2025high} tests the weighted loss with a concentration-based p-value and uses every calibration point; the clipped weighted conformal prediction of \citet{wang2026weight} learns a clipped ratio and inflates the target by an estimated clipping bias. Both are discussed in the related work below and compared with our certificate in Sections~\ref{sec:clip} and \ref{sec:experiments}.

This paper shows that the standard weighted quantile itself, which is deterministic and uses every calibration point with its weight, admits an explicit training-conditional certificate, and that the certificate is governed by a tail average of the likelihood ratio rather than by its supremum.

\paragraph{Contributions.}
\begin{itemize}[leftmargin=1.5em,itemsep=2pt]
\item \textbf{Explicit training-conditional bounds and PAC calibration (Section~\ref{sec:upper}).} A threshold that under-covers must cross a fixed population quantile, and the event that it does is a one-sided deviation of a single bounded i.i.d.\ sum; one concentration inequality therefore suffices, with no uniform convergence and no unspecified constant. The resulting bounds (Theorem~\ref{thm:upper}) are the exchangeable ones with $m$ replaced by an effective calibration size that is $m/B$ in the worst case but $m/\mathrm{CVaR}_\alpha(v)$ in general, where $\mathrm{CVaR}_\alpha(v)$ is the average of the likelihood ratio over the $\alpha$-fraction of the test population where it is largest; the supremum $B$ enters only lower-order terms. Corollary~\ref{cor:pac} and Proposition~\ref{prop:bentkus} turn them into a deterministic $(\alpha,\delta)$-PAC procedure, summarized in Procedure~\ref{proc:wq}.
\item \textbf{Lower bounds and the role of score--weight dependence (Section~\ref{sec:lower}).} For any covariate shift and any threshold rule, the training-conditional coverage deviates from $1-\alpha$ by $\Omega(\sqrt{(1+\chi^2(Q\|P))\alpha(1-\alpha)/m})$ on some score distribution, so the $\chi^2$-divergence, not the supremum, is the intrinsic price of the shift. For the weighted quantile itself the full variance proxy $\sigma_\alpha^2$, including its coupling term, is the exact asymptotic variance under the adversarial alignment of the score tail with the high-ratio region.
\item \textbf{Clipping, estimated inputs and empirical comparisons (Sections~\ref{sec:estimated}--\ref{sec:experiments}).} Clipping the ratio trades a computable bias for a smaller range, for the weighted quantile and for rejection sampling alike, and the clip level can be optimized before any label is seen; this yields explicit crossover conditions between the two constructions. An oracle inequality handles estimated ratios, an exact guarantee handles interval-valued ratios, and simultaneous confidence bounds from an unlabeled source sample make the certificate fully finite-sample when the tail functionals of $v$ are unknown. Experiments on synthetic shifts and four UCI datasets locate each valid procedure on the resulting frontier.
\end{itemize}

\paragraph{Related work.}
Training-conditional coverage in the exchangeable case goes back to \citet{vovk2012conditional}; \citet{bian2023training} treat split conformal, CV+ and jackknife+, \citet{liang2023algorithmic} use algorithmic stability, and \citet{duchi2025sample} gives Bernstein-type and conditional versions and points to the minimax lower bounds of \citet{areces2024two}. \citet{park2022pac} construct PAC prediction sets under covariate shift by rejection sampling the calibration set down to an exchangeable subsample of expected size $m/B$ and applying a Clopper--Pearson bound; Theorem~\ref{thm:lower} shows that this sample-size loss is unavoidable in the worst case $v\in\{0,B\}$, where $K^2=B$, while for a general shift the intrinsic scale is $K^2=1+\chi^2(Q\|P)$, which can be much smaller than $B$, and Theorem~\ref{thm:upper} shows that the deterministic weighted quantile attains the $m^{-1/2}$ rate with explicit constants (the matching is in $m$ and in the $K^2$ scale, not uniformly in $\alpha$; Section~\ref{sec:lower}). \citet{almeida2025high} weight the calibration losses by the likelihood ratio inside the learn-then-test framework of \citet{angelopoulos2025learn}: since $\EE_P[v(X)L(X,Y;\lambda)]$ is the target risk, any valid p-value for the mean of the weighted loss, which lies in $[0,B]$, can be used in a fixed-sequence test over thresholds; they use the betting p-value of \citet{waudby2024estimating}, with Hoeffding--Bentkus and empirical Bernstein p-values as alternatives. The procedure is deterministic (up to the data ordering of the betting p-value) and its empirical-variance p-values adapt to the variance of the weighted loss,
\[
\Var_P\big(v\ind{S>\lambda}\big)=\EE_Q\big[v\ind{S>\lambda}\big]-Q(S>\lambda)^2,
\]
which shares its coupling moment with \eqref{eq:secondmoment_main} but evaluated at the actual coupling rather than at the worst case; the range $B$ of the weighted loss enters both concentration-based p-values, whereas the summands of Theorem~\ref{thm:upper} have one-sided range $\alpha B$. Section~\ref{sec:experiments} compares the two concentration-based variants with our certificate; the betting variant is not evaluated. \citet{wang2026weight} study weighted conformal prediction with clipped likelihood ratios that may be unbounded and must be learned: they fit the clipped ratio $\min(w,B)$ directly, estimate the clipping bias $\EE_P(w-B)^+=1-\EE_P\min(w,B)$ from unlabeled source data, and prove two-sided dataset-conditional coverage for the weighted quantile at an inflated target through a weighted DKW inequality, with calibration sizes of order $B\log(1/(\epsilon\delta))/\epsilon^2+B^2\log(1/\delta)/\epsilon^2$ for an additional inflation $\epsilon$. Their clipping bias is the quantity $1-\mu_{B'}$ of Lemma~\ref{lem:clipbias}, which dominates our $\Delta_{B'}$, and their bias estimate plays the role of $r_+$ in Proposition~\ref{prop:est}; for a known clipped ratio, Proposition~\ref{prop:clip}(i) replaces the DKW-type sample complexity by a Bernstein-type inflation that is feasible at the calibration sizes of Section~\ref{sec:experiments}, and Section~\ref{sec:clip} optimizes the clip level. The estimation of the ratio itself is complementary to our analysis, which takes $v$ as given or as a fixed estimate. For estimated weights, \citet{lei2021conformal} bound the marginal coverage loss by $\frac12\EE_P|\widehat w-w|$, and \citet{barber2023conformal} bound the coverage gap of any fixed-weight quantile by a weighted total-variation distance; \citet{qiu2023prediction} and \citet{yang2024doubly} obtain asymptotic validity under unknown shift through doubly robust constructions. Theorem~\ref{thm:estimated} is the training-conditional counterpart of the first two results.

\section{Setting}\label{sec:setting}

Throughout we condition on the training split, so the score function $s:\cX\times\cY\to\R$ is fixed. The calibration sample $Z_i=(X_i,Y_i)$, $i\in\cI$, $|\cI|=m$, is i.i.d.\ from $P=P_X\times P_{Y|X}$ and the test point $Z_{\rm test}\sim Q=Q_X\times P_{Y|X}$. We assume $Q\ll P$ and write
\[
v(z):=\frac{dQ}{dP}(z)=\frac{dQ_X}{dP_X}(x),\qquad v_i:=v(Z_i),\qquad S_i:=s(Z_i);
\]
until Section~\ref{sec:estimated} we assume that $v$ is known and $v\le B$ for a known $B\ge1$. Let
\[
F_Q(t):=\PP_{Z\sim Q}(s(Z)\le t),\qquad q_Q(\beta):=\inf\{t:\ F_Q(t)\ge\beta\},\qquad P_e(\tau):=1-F_Q(\tau),
\]
so that $F_Q(q_Q(\beta))\ge\beta$ and $\PP_Q(S<q_Q(\beta))\le\beta$ by right-continuity. The self-normalized weighted empirical CDF of the calibration scores and its $\beta$-quantile are
\begin{equation}\label{eq:tauhat}
\begin{gathered}
\Fhat(t):=\sum_{i\in\cI}\what_i\ind{S_i\le t},\qquad \what_i:=\frac{v_i}{\sum_{j\in\cI}v_j},\\
\tauhat(\beta):=\inf\{t:\ \Fhat(t)\ge\beta\}\in\R\cup\{+\infty\},
\end{gathered}
\end{equation}
with $\Fhat\equiv0$ if $\sum_jv_j=0$. Since $\Fhat$ is a right-continuous step function, $\Fhat(\tauhat(\beta))\ge\beta$ whenever $\tauhat(\beta)<\infty$. The prediction set
\[
\widehat C_\beta(x):=\{y:\ s(x,y)\le\tauhat(\beta)\}
\]
depends on the calibration data only, so its training-conditional miscoverage is exactly $P_e(\tauhat(\beta))$.

The weighted split conformal set of \citet{tibshirani2019conformal} at level $1-\alpha$ uses the threshold $F_{Q,m}^{-1}(1-\alpha)$, where
\[
F_{Q,m}(t)=w_{\rm test}\ind{t=\infty}+\sum_iw_i\ind{S_i\le t},\qquad w_i=\frac{v_i}{v(Z_{\rm test})+\sum_jv_j}\le\what_i.
\]
Hence $F_{Q,m}\le\Fhat$ on $\R$ for every value of $v(Z_{\rm test})$, the threshold is at least $\tauhat(1-\alpha)$, and
\begin{equation}\label{eq:reduction}
P_e^{\rm split}(D):=\PP\big(Y_{\rm test}\notin\widehat C^{\rm split}_\alpha(X_{\rm test})\mid D\big)\le P_e\big(\tauhat(1-\alpha)\big).
\end{equation}
Everything below is a statement about the random variable $P_e(\tauhat(\beta))$; every upper bound on its miscoverage transfers to the original weighted split conformal set through \eqref{eq:reduction} (bounds on over-coverage do not); the set $\widehat C_\beta$ is also the natural object in practice, since its threshold is common to all test points (the set $\widehat C_\beta(x)$ still depends on $x$ through the score) and does not require the test weight.

\section{Training-conditional bounds without uniform convergence}\label{sec:upper}

\paragraph{Why one inequality suffices.} Fix a level $\beta=1-\alpha$ and suppose the weighted quantile under-covers by more than $\gamma$: $P_e(\tauhat(\beta))>\alpha+\gamma$. Then $\tauhat(\beta)$ lies strictly below the \emph{fixed} population quantile $q:=q_Q(\beta-\gamma)$, and since $\Fhat$ reaches $\beta$ at $\tauhat(\beta)$, the weighted empirical mass strictly below $q$ is at least $\beta$:
\[
\sum_{i\in\cI}\what_i\ind{S_i<q}\ \ge\ \beta .
\]
Clearing the denominator, this is
\[
\frac1m\sum_{i\in\cI}U_i\ \ge\ 0,\qquad U_i:=v_i\big(\ind{S_i<q}-\beta\big),
\]
where the $U_i$ are i.i.d., bounded by $B$ in absolute value, with mean $\PP_Q(S<q)-\beta\le-\gamma$ (because $\EE_P[v\,h(Z)]=\EE_Q[h(Z)]$). Under-coverage by $\gamma$ thus forces a bounded i.i.d.\ sum to exceed its mean by at least $\gamma$, and a single concentration inequality bounds the probability of that event. No uniformity over thresholds is needed, and the self-normalizing denominator $\sum_jv_j$ never has to be controlled separately; this is the observation behind Vovk's training-conditional guarantee in the form given by \citet[Sec.~2.1]{duchi2025sample}, applied with the weights kept inside the summands. The rate is decided by the second moment of $U_i$,
\begin{equation}\label{eq:secondmoment_main}
\EE U_i^2=\alpha^2\,\EE_Q\big[v\ind{S<q}\big]+(1-\alpha)^2\,\EE_Q\big[v\ind{S\ge q}\big]=\alpha^2K^2+(1-2\alpha)\,\EE_Q\big[v\ind{S\ge q}\big],
\end{equation}
where $K^2:=\EE_Q[v]$. The second term is the likelihood ratio integrated over the (roughly) $\alpha$-fraction of the test population that falls outside the prediction set. Bounding it by $B\cdot\PP_Q(S\ge q)$ gives a bound scaled by $B$; bounding it instead by its largest possible value over all subsets of test probability about $\alpha$ gives a bound scaled by a tail average of $v$. This is where the supremum and the tail average part ways.

\paragraph{The functionals.} Let $F_v$ be the CDF of $v(Z)$ for $Z\sim Q$ and $F_v^{-1}$ its quantile function. For $a\in(0,1]$ define
\begin{equation}\label{eq:tailfunctionals}
K^2:=\EE_Q[v]=\EE_P[v^2]=1+\chi^2(Q\|P),\qquad \rho_a:=F_v^{-1}(1-a),
\end{equation}
\begin{equation}\label{eq:Lambda}
\Lambda(a):=\int_{1-a}^{1}F_v^{-1}(u)\,du=\sup\big\{\EE_Q[vh]:\ 0\le h\le1,\ \EE_Qh\le a\big\},
\end{equation}
where the last identity is the Hardy--Littlewood inequality (the supremum is over $[0,1]$-valued $h$ rather than over events so that it remains an equality when $v$ has atoms under $Q$). In words: $K^2$ is the second moment of the ratio under $P$, one plus the $\chi^2$-divergence of the shift; $\Lambda(a)/a=:\mathrm{CVaR}_a(v)$ is the average of the likelihood ratio over the $a$-fraction of the test population where it is largest, and $\rho_a$ the corresponding value at risk, the ratio at the boundary of that fraction. Clearly
\[
aK^2\ \le\ \Lambda(a)\ \le\ \min(aB,K^2),\qquad \rho_a\le B,\qquad K^2\le B,
\]
and $\Lambda$ is concave. The variance proxy that will govern the weighted quantile is
\[
\sigma_\alpha^2:=\alpha^2K^2+(1-2\alpha)\Lambda(\alpha),\qquad \alpha(1-\alpha)K^2\ \le\ \sigma_\alpha^2\ \le\ \alpha(1-\alpha)B\quad(\alpha\le1/2),
\]
with equality on the right iff $v\in\{0,B\}$ $Q$-a.s. In the running example of the introduction ($v=20.99$ on a subgroup of test probability $1.05\%$ and $0.99$ elsewhere, $\alpha=0.05$), $K^2=1.20$, $\Lambda(\alpha)=0.26$ so that $\mathrm{CVaR}_\alpha(v)=5.2$ against $B=21$, $\rho_\alpha=0.99$, and $\sigma_\alpha^2=0.24$ against $\alpha(1-\alpha)B=1.0$. All three functionals are population quantities of the law of $v$ under $Q$; since $\EE_Q[g(v)]=\EE_P[v\,g(v)]$, they can be evaluated from $v$ and the source covariate distribution, and the results below treat them, like $B$, as known (Remark~\ref{rem:known} and Section~\ref{sec:est} address their estimation).

\begin{theorem}\label{thm:upper}
Assume $Q\ll P$ and $v=dQ/dP\le B$. Fix $\beta\in(0,1)$, put $\alpha:=1-\beta$, and let $\gamma>0$. Then
\begin{align}
\PP\big(P_e(\tauhat(\beta))>\alpha+\gamma\big)&\le \exp\big(-2m\gamma^2/B^2\big), \label{eq:hoeff}\\
\PP\big(P_e(\tauhat(\beta))>\alpha+\gamma\big)&\le \exp\Big(-\frac{m\gamma^2}{2B\alpha(1-\alpha)+\frac83 B\gamma}\Big), \label{eq:bern}
\end{align}
and, if $\alpha\le1/2$, with $\sigma_\alpha^2:=\alpha^2K^2+(1-2\alpha)\Lambda(\alpha)$,
\begin{equation}\label{eq:tail}
\PP\big(P_e(\tauhat(\beta))>\alpha+\gamma\big)\le \exp\Big(-\frac{m\gamma^2}{2\sigma_\alpha^2+\big(2(1-2\alpha)\rho_\alpha+\frac23\alpha B\big)\gamma+\frac23\gamma^2}\Big).
\end{equation}
Consequently, for every $\delta\in(0,1)$, with $L:=\log(1/\delta)$,
\begin{equation}\label{eq:gamma_m}
\PP\big(P_e(\tauhat(\beta))>\alpha+\gamma_m(\delta)\big)\le\delta,\qquad
\gamma_m(\delta):=\min\Big\{B\sqrt{\tfrac{L}{2m}},\ \tfrac{4BL}{3m}+\sqrt{\big(\tfrac{4BL}{3m}\big)^2+\tfrac{2B\alpha(1-\alpha)L}{m}}\Big\},
\end{equation}
and the second expression is at most $\frac{8BL}{3m}+\sqrt{2B\alpha(1-\alpha)L/m}$.
\end{theorem}

The proof (Appendix~\ref{app:upper}) is the argument above: Hoeffding's inequality for the sum of the $U_i$ gives \eqref{eq:hoeff}; Bernstein's inequality with the second moment bounded through $B$ gives \eqref{eq:bern}, and with the second term of \eqref{eq:secondmoment_main} bounded by $\Lambda(\PP_Q(S\ge q))\le\Lambda(\alpha)+(\text{deficit})\cdot\rho_\alpha$ gives \eqref{eq:tail}. The supremum $B$ survives only in the lower-order linear term of \eqref{eq:tail}, through the one-sided range $\alpha B$ of $U_i$.

\begin{remark}[Comparison with \eqref{eq:old}]
Inequality \eqref{eq:hoeff} removes the constant $C$, improves the exponent by a factor $4$ and the prefactor from $4$ to $1$. Inequality \eqref{eq:bern} replaces the factor $B$ in the leading term by $\sqrt{B\alpha(1-\alpha)}$, which is what makes small $\alpha$ affordable: for $\alpha=\delta=0.05$ and $B=4$, the Hoeffding inflation is below $\alpha$ (so that the inflated set is not all of $\cY$) only for $m\ge9600$, the Bernstein inflation for $m\ge1100$. For $B=1$ the two bounds are Vovk's $e^{-2m\gamma^2}$ and Proposition~2 of \citet{duchi2025sample}.
\end{remark}

\begin{remark}[Which scale is the right one]\label{rem:scale}
In the worst case $v\in\{0,B\}$, $\Lambda(\alpha)=\alpha B$ and \eqref{eq:tail} reduces to \eqref{eq:bern} with a slightly better linear term. In general, bounds \eqref{eq:bern} and \eqref{eq:tail} read as the exchangeable bounds of \citet{vovk2012conditional,duchi2025sample} with $m$ replaced by an effective calibration size that is $m/B$ in the worst case but $m/\mathrm{CVaR}_\alpha(v)$ in general, where $\mathrm{CVaR}_\alpha(v)=\Lambda(\alpha)/\alpha$. The two scales differ when the likelihood ratio is large only on a set of test probability well below $\alpha$: a rare subgroup that is over-represented at test time but still makes up less than $\alpha$ of it, or a ratio with a heavy right tail, as for a Gaussian mean shift. Theorem~\ref{thm:lower} shows that $K^2=1+\chi^2(Q\|P)$ is a lower bound for any threshold rule, and Remark~\ref{rem:coupling} that $\Lambda(\alpha)$ cannot be improved for the weighted quantile without assumptions on how the score tail and the likelihood ratio are coupled. The rejection-sampling construction of \citet{park2022pac}, which accepts calibration points with probability $v_i/B$, is tied to the supremum and pays $m/B$ regardless (Section~\ref{sec:exp_rare}).
\end{remark}

\begin{remark}[Known functionals and their estimation]\label{rem:known}
Theorem~\ref{thm:upper} and everything built on it take $B$, $K^2$, $\Lambda(\alpha)$ and $\rho_\alpha$ as known, exactly as \citet{park2022pac} take $B$ as known. Knowing the function $v$ does not by itself make these expectations known: they are functionals of the law of $v$ under $P_X$. There are two settings in which they are. If $P_X$ is a known or fully specified distribution---in Section~\ref{sec:exp_real}, the empirical distribution of a finite pool, with $v$ normalized on that pool---they can be computed exactly. Otherwise they can be estimated from an unlabeled source sample independent of the calibration data, and then the plug-in certificate is exact only up to that estimation error; since $v\le B$, Hoeffding-type confidence bounds for $K^2=\EE_P[v^2]$, $\mu_{B'}=\EE_P[\min(v,B')]$ and $\EE_P[v^2\ind{v>\rho}]$ have ranges $B^2$ and $B$, so a rigorous version replaces each functional by an upper confidence bound at level $\delta'$ and charges $\delta'$ to the failure probability; Proposition~\ref{prop:est} in Section~\ref{sec:est} does exactly this, and Section~\ref{sec:exp_est} measures the cost. An empirical maximum of $v$ is not a certified supremum, and $B$ is taken as known throughout. In the simulations of Sections~\ref{sec:exp_smooth}--\ref{sec:exp_frontier} the functionals are either exact (closed form or one-dimensional quadrature for the synthetic shifts, Appendix~\ref{app:experiments}) or computed from a large source sample when $B$ is small; those certificates are population statements, and Section~\ref{sec:exp_est} shows what a single source sample of realistic size delivers.
\end{remark}

\begin{corollary}[PAC calibration and two-sided control]\label{cor:pac}
Assume $Q\ll P$, $v\le B$, $\alpha\in(0,1/2]$, and let $\gamma\in(0,\alpha)$. Run the weighted quantile at the inflated level $\beta:=1-\alpha+\gamma$. Then
\begin{enumerate}[label=(\roman*),leftmargin=2em,itemsep=0pt]
\item $\displaystyle\PP\big(P_e(\tauhat(\beta))>\alpha\big)\le\exp\Big(-\frac{m\gamma^2}{2\alpha^2K^2+2(1-2\alpha+2\gamma)\Lambda(\alpha)+\frac23(\alpha B+\gamma)\gamma}\Big)$,\\[2pt] and the right-hand side is at most $\exp\{-m\gamma^2/(2B\alpha(1-\alpha)+3B\gamma)\}$;
\item if $F_Q$ is continuous, then for every $\gamma'\in(0,\alpha-\gamma)$, \\[2pt] $\displaystyle\PP\big(P_e(\tauhat(\beta))<\alpha-\gamma-\gamma'\big)\le\exp\Big(-\frac{m\gamma'^2}{2\sigma_\alpha^2+4\gamma\Lambda(\alpha)+\frac23\big((1-\alpha+\gamma)B+\gamma'\big)\gamma'}\Big)$.
\end{enumerate}
In particular, if $m>\frac23L$ and $\gamma=\gamma_m(\delta)$ is chosen so that the first bound in (i) equals $\delta$, namely
\begin{equation}\label{eq:gamma_tail}
\gamma_m(\delta)=\frac{L\,b+\sqrt{L^2b^2+8(m-\tfrac23L)L\sigma_\alpha^2}}{2(m-\tfrac23L)},\qquad b:=4\Lambda(\alpha)+\tfrac23\alpha B,\quad L=\log(1/\delta),
\end{equation}
(when $\gamma_m(\delta)\ge\alpha$ the inflated level exceeds one and the certified set is all of $\cY$, which is valid but uninformative; this is the ``infeasible'' case in the experiments)
then the $(1-\alpha)$-coverage guarantee holds for a $(1-\delta)$-fraction of calibration sets; and if $\gamma'_m(\delta)$ is chosen so that the bound in (ii) equals $\delta$ (a quadratic in $\gamma'$ whose linear coefficient is $\frac23(1-\alpha+\gamma)B$ instead of the $\frac23\alpha B$ of (i)), then $\alpha-\gamma_m(\delta)-\gamma'_m(\delta)\le P_e\le\alpha$ with probability at least $1-2\delta$. The two radii differ: the over-coverage radius $\gamma'_m(\delta)$ has the full range $(1-\alpha+\gamma)B$ in its linear term, because a single calibration point with a large weight and a large score can push the threshold up by a lot, while it cannot push it down; the one-sided PAC radius only sees the range $\alpha B$. Both are of order $\sqrt{\sigma_\alpha^2\log(1/\delta)/m}$ once the linear terms are negligible, which for the over-coverage radius requires $m\gg B^2\log(1/\delta)/\sigma_\alpha^2$ (not merely $m\gg B\log(1/\delta)$): this is exactly the regime of a large supremum with a small variance proxy. If the quadratic in (ii) has no root in $(0,\alpha-\gamma)$, statement (ii) gives only the trivial bound $P_e\ge0$. Statement (i) also holds for the set of \citet{tibshirani2019conformal} run at level $\beta$, by \eqref{eq:reduction}.
\end{corollary}

\begin{center}
\fbox{\begin{minipage}{0.96\linewidth}\small
\refstepcounter{procedure}\label{proc:wq}\textbf{Procedure \theprocedure\ (PAC-calibrated weighted quantile).}
\emph{Inputs:} calibration scores $S_i$ and weights $v_i=v(Z_i)$, $i\in\cI$; target miscoverage $\alpha\le1/2$ and confidence $1-\delta$; the supremum $B$ and the population functionals $K^2$, $\Lambda(\alpha)$ (and $\rho_\alpha$ for the Bentkus inflation), taken as known.
\begin{enumerate}[leftmargin=1.6em,itemsep=0pt,topsep=2pt]
\item If $m\le\frac23\log(1/\delta)$, return $\widehat C\equiv\cY$: the certificate is infeasible at this calibration size.
\item Compute the inflation $\gamma=\gamma_m(\delta)$ from \eqref{eq:gamma_tail}, or $\gamma_m^{\rm Bk}(\delta)$ from Proposition~\ref{prop:bentkus}; if $\gamma\ge\alpha$, return $\widehat C\equiv\cY$.
\item Otherwise let $\tauhat$ be the self-normalized weighted quantile \eqref{eq:tauhat} at level $1-\alpha+\gamma$ and return $\widehat C(x)=\{y:s(x,y)\le\tauhat\}$.
\end{enumerate}
\emph{Guarantee:} $\PP\big(P_e(\widehat C)>\alpha\big)\le\delta$ over the calibration sample (Corollary~\ref{cor:pac}(i), Proposition~\ref{prop:bentkus}). \emph{Clipped variant} (Proposition~\ref{prop:clip}): return $\widehat C\equiv\cY$ if $\Delta_{B'}\ge\alpha$; otherwise replace $v$ by $\min(v,B')$, $\alpha$ by $\alpha-\Delta_{B'}$ and the functionals by those of the normalized clipped ratio; the clip level is chosen from population quantities, before any label is seen. \emph{Estimated functionals:} when $K^2$, $\Lambda(\alpha)$, $\rho_\alpha$ or $\Delta_{B'}$ are estimated from an unlabeled source sample, use Proposition~\ref{prop:est} instead, which reserves a budget $\delta'$ for the estimation, computes the inflation at level $\delta-\delta'$, and guarantees $\PP(P_e>\alpha)\le\delta$ jointly over the source and calibration samples.
\end{minipage}}
\end{center}

Part (i) is Theorem~\ref{thm:upper} with the fixed quantile placed at $q_Q(1-\alpha)$ instead of $q_Q(\beta-\gamma)$, which removes the $\rho_\alpha$ term; part (ii) is the mirror-image argument: coverage above $\beta+\gamma$ forces $\tauhat(\beta)>q_Q(\beta+\gamma)$, which forces $\Fhat(q_Q(\beta+\gamma))<\beta$, a lower-tail event for a sum of bounded i.i.d.\ variables with positive mean $\gamma$. The inflation \eqref{eq:gamma_tail} is at most
\[
\frac{Lb}{m-2L/3}+\sqrt{\frac{2L\sigma_\alpha^2}{m-2L/3}}\,;
\]
in the worst case $v\in\{0,B\}$ it is essentially the Bernstein expression in \eqref{eq:gamma_m}, and we refer to both as the \emph{Bernstein inflation}, qualified by ``sup $B$'' or ``tail-adaptive'' when the distinction matters. The Bernstein inequality can be replaced by Bentkus's inequality \citep{bentkus2004hoeffding}, which compares the tail of a sum of bounded variables with a log-linearly interpolated binomial tail; it is not uniformly sharper than Bernstein's inequality, but it gives smaller inflations in all our configurations. Let $\mathrm{Bk}_m(x;\sigma^2,b)$ denote the right-hand side of Bentkus's Theorem~1.1: for independent $X_i$ with $\EE X_i=0$, $X_i\le b$ and $\Var X_i\le\sigma^2$,
\[
\PP\Big(\sum_{i=1}^mX_i\ge x\Big)\le\mathrm{Bk}_m(x;\sigma^2,b):=\tfrac{e^2}{2}\,\PP^\circ\Big(\sum_{i=1}^m\varepsilon_i\ge x\Big),
\]
where the $\varepsilon_i$ are i.i.d.\ two-point variables with
\[
\PP(\varepsilon_i=b)=\frac{\sigma^2}{b^2+\sigma^2},\qquad \PP(\varepsilon_i=-\sigma^2/b)=\frac{b^2}{b^2+\sigma^2},
\]
and $\PP^\circ$ is the log-concave hull of the binomial survival function (log-linear interpolation between integers).

\begin{proposition}[Bentkus inflation]\label{prop:bentkus}
In the setting of Corollary~\ref{cor:pac}, for $\beta=1-\alpha+\gamma$,
\[
\PP\big(P_e(\tauhat(\beta))>\alpha\big)\ \le\ \sup_{t\in[\gamma,\,\beta]}\ \mathrm{Bk}_m\big(mt;\ \sigma^2(t),\ (\alpha-\gamma)B+t\big),
\]
where $\sigma^2(t):=(\alpha-\gamma)^2K^2+(1-2\alpha+2\gamma)\big[\Lambda(\alpha)+(t-\gamma)\rho_\alpha\big]$.
Consequently every $\gamma$ for which the right-hand side is at most $\delta$ is a valid inflation, and we write $\gamma_m^{\rm Bk}(\delta)$ for the smallest one. The right-hand side need not be monotone in $\gamma$ (an example with $\alpha=0.02$, $B=2$, $m=100$ and a two-point ratio has a feasible $\gamma$ followed by an infeasible larger one at $\delta=0.97$), so a bisection returns a valid inflation but is guaranteed to return the smallest one only when the bound is monotone on $[0,\alpha]$; in every configuration of Table~\ref{tab:frontier} we verified numerically that no smaller inflation on a fine grid satisfies the implemented interval bound of Appendix~\ref{app:experiments}.
\end{proposition}

The proof is that of Corollary~\ref{cor:pac}(i) with Bentkus's inequality applied at the unknown deficit $t=\beta-\PP_Q(S<q)\in[\gamma,\beta]$ in place of Bernstein's, the supremum over $t$ replacing the monotonicity argument. In the configurations of Section~\ref{sec:exp_frontier} the supremum is attained at $t=\gamma$, but not in general (it moves to the right when $\rho_\alpha$ is large relative to $\Lambda(\alpha)$), so the implementation certifies it over a partition of $[\gamma,\beta]$ rather than on a grid of points (Appendix~\ref{app:experiments}); $\gamma_m^{\rm Bk}(\delta)$ is $9$--$12\%$ smaller than \eqref{eq:gamma_tail} across the configurations of Section~\ref{sec:exp_frontier} in which both inflations are below $\alpha$ ($8.8\%$ for $B=51$, $m=20000$, $\delta=0.01$; $12.4\%$ for $B=51$, $m=2000$, $\delta=0.05$), and the Bentkus inflation is feasible in two configurations where the Bernstein one is not (there the formula-level reduction is larger, e.g.\ $15\%$ for the Gaussian shift at $m=2000$, $\delta=0.01$); in the units of the Gaussian benchmark, the Bernstein inflation corresponds to about $2.5$ standard deviations of $P_e$, the Bentkus inflation to about $2.25$, and the oracle to $z_{1-\delta}=1.65$. We record two further variants of Theorem~\ref{thm:upper} whose proofs are equally short (Appendix~\ref{app:upper}).

\begin{remark}[Weaker moment assumptions]\label{rem:moments}
If only $\|v\|_{L^2(P)}\le K$ is assumed, Chebyshev's inequality in place of Hoeffding's gives
\[
\PP\big(P_e(\tauhat(\beta))>\alpha+K/\sqrt{m\delta}\big)\le\delta,
\]
replacing the $O(1/(\delta\sqrt m))$ rate of Corollary~1 in \citet{pournaderi2026training} by $O(1/\sqrt{\delta m})$. If $v$ is unbounded, truncating $v$ at any level $B$ inside $U_i$ gives, with
\[
r_B:=\EE_P(v-B)^+=\EE_Q(1-B/v)^+\le\PP_Q(v>B)
\]
and any $\kappa\in(0,\gamma-r_B)$,
\[
\PP\big(P_e(\tauhat(\beta))>\alpha+\gamma\big)\le\exp\Big(-\frac{2m(\gamma-r_B-\kappa)^2}{B^2}\Big)+\frac{r_B}{\kappa},
\]
a considerably simpler statement than Theorem~2 of \citet{pournaderi2026training}.
\end{remark}

\section{Lower bounds and sharpness}\label{sec:lower}

Theorem~\ref{thm:upper} has two scales in it, and they are sharp in different senses. The scale $K^2=1+\chi^2(Q\|P)$ at the lower end of the range of $\sigma_\alpha^2$ is intrinsic to the shift: Theorem~\ref{thm:lower} shows that no threshold rule, deterministic or randomized, can deviate from the nominal coverage by less than $\Omega(\sqrt{K^2\alpha(1-\alpha)/m})$ on every score distribution, so for $v\in\{0,B\}$, where $K^2=B$, the upper bound is matched exactly. The coupling term $\Lambda(\alpha)$, which lifts $\sigma_\alpha^2$ above $\alpha(1-\alpha)K^2$ when the score tail sits on the high-ratio region, is sharp for the weighted quantile itself: Proposition~\ref{prop:clt} shows that $\sigma_\alpha^2$ is its exact asymptotic variance under that adversarial coupling. Whether other procedures can avoid the coupling term is open, so ``matching'' below means matching in $m$ and in the $K^2$ scale, not uniformly in $\alpha$.

A \emph{threshold rule} is a measurable map $\tauhat=\tauhat(D,\xi)\in\R\cup\{\pm\infty\}$ of the calibration sample and of an independent randomization variable $\xi$; it may use the known likelihood ratio, and its miscoverage is $P_e(\tauhat)=1-F_Q(\tauhat)$.

\begin{theorem}\label{thm:lower}
Let $\alpha\in(0,1/2]$. Let $(P_X,Q_X)$ be any pair of covariate distributions with $v=dQ_X/dP_X\le B$ and $K^2=\EE_Q[v]$, let $m\ge B^2/(4\alpha K^2)$, and put
\[
\gamma_m^{\rm lb}:=\frac{1}{4\sqrt2}\sqrt{\frac{K^2\alpha(1-\alpha)}{m}}.
\]
For every threshold rule $\tauhat$ there exist two conditional laws $P^{(0)}_{S|X},P^{(1)}_{S|X}$ of the score given the covariate, with bounded densities, defining the laws $P_j$ and $Q_j$ of $(X,S)$ with covariate marginals $P_X$ and $Q_X$, such that
\begin{equation}\label{eq:lecam}
\PP_{P_0}\big(P_e^{(0)}(\tauhat)\le\alpha-\gamma_m^{\rm lb}\big)+\PP_{P_1}\big(P_e^{(1)}(\tauhat)\ge\alpha+\gamma_m^{\rm lb}\big)\ \ge\ \tfrac12,
\end{equation}
where $P_e^{(j)}(\tau)=1-F_{Q_j}(\tau)$ and $\PP_{P_j}$ refers to $D\sim P_j^{\otimes m}$. In particular, (a) for some $j$, $|P_e^{(j)}(\tauhat)-\alpha|\ge\gamma_m^{\rm lb}$ with probability at least $1/4$; and (b) if $\tauhat$ is $(\alpha,\delta)$-PAC for both pairs with $\delta\le1/4$, then $\PP_{P_0}\big(P_e^{(0)}(\tauhat)\le\alpha-\gamma_m^{\rm lb}\big)\ge\frac12-\delta$.
\end{theorem}

The theorem is stated for the conditional law of the score, which is all a threshold rule sees; with the identity score $s(x,y)=y$ these are conditional laws of $Y$. For a fixed general score map the theorem applies whenever the two conditional score laws of the proof, which have bounded densities on a common interval for every $x$, can be realized by conditional laws of $Y$; a sufficient condition is that $s(x,\cdot)$ be continuous with a range containing a fixed interval $I$ for every $x$, in which case a measurable right inverse of $s(x,\cdot)$ on $I$ exists and the law of $S$ given $X=x$ can be any law on $I$. A common interval is essential and not merely convenient: if the ranges of $s(x,\cdot)$ for different $x$ are disjoint (say $s(x,y)=2x+(1+e^{-y})^{-1}$ with $X\in\{0,1\}$ and $Q_X(X=1)=\alpha$), a fixed threshold can have miscoverage exactly $\alpha$ for every conditional law of $Y$, and no lower bound of this kind can hold. The construction (Appendix~\ref{app:lower}) keeps the covariate shift as given and perturbs the conditional score law at $x$ by an amount proportional to $v(x)$, moving mass from below to above the $(1-\alpha)$-quantile; localizing the perturbation at the quantile produces the factor $\alpha(1-\alpha)$ in the $\chi^2$-distance, weighting it by $v$ produces $K^2=1+\chi^2(Q\|P)$, and a two-point (Le Cam) argument finishes the proof. For $v\in\{0,B\}$, $K^2=B$ and the bound reads $\Omega(\sqrt{B\alpha(1-\alpha)/m})$, matching Theorem~\ref{thm:upper} exactly; in general it matches the lower end $\alpha(1-\alpha)K^2$ of the range of $\sigma_\alpha^2$ in Remark~\ref{rem:scale}. Part (b) is the statement that matters for Corollary~\ref{cor:pac}: any rule that certifies coverage for most calibration sets must, on some distribution, over-cover by $\Omega(\sqrt{K^2\alpha(1-\alpha)/m})$ with constant probability. Theorem~1 of \citet{areces2024two} gives the analogous unshifted rate for a different procedure class and loss; rules with $x$-dependent thresholds are covered as soon as $Q_X$ is a point mass, and in general for thresholds that do not depend on $x$.

\begin{remark}[The coupling term $\Lambda(\alpha)$]\label{rem:coupling}
The gap between $\alpha(1-\alpha)K^2$ and $\sigma_\alpha^2$ is real for the weighted quantile itself. Proposition~\ref{prop:clt} below shows that the asymptotic variance of $\sqrt m\,(P_e(\tauhat(1-\alpha))-\alpha)$ is exactly the second moment \eqref{eq:secondmoment_main} evaluated at the true coupling,
\[
\alpha^2K^2+(1-2\alpha)\,\EE_Q\big[v\ind{S>q_Q(1-\alpha)}\big],
\]
which equals $\sigma_\alpha^2$ when the upper $\alpha$-tail of the score under $Q$ is the set on which $v$ is largest, and $\alpha(1-\alpha)K^2$ when $v$ is independent of the score. The variance proxy of Theorem~\ref{thm:upper}(c) is therefore the exact asymptotic variance of the weighted quantile's coverage under the adversarial coupling, so the leading variance scale of \eqref{eq:tail} (not its constants or lower-order terms) cannot be improved for the weighted quantile without assumptions on how the score tail and the likelihood ratio are coupled. Whether \emph{other} procedures can do better in this coupling is a different question: the lower bound of Theorem~\ref{thm:lower} is at the scale $K^2\alpha(1-\alpha)$, and we do not claim minimax optimality of the coupling term $\Lambda(\alpha)$. The rate $\sqrt{K^2\alpha(1-\alpha)/m}$ is attained by the weighted quantile when the score tail is not concentrated on the high-ratio region, which is the situation observed in Section~\ref{sec:exp_rare}.
\end{remark}

\begin{proposition}[Asymptotic variance of the weighted quantile's coverage]\label{prop:clt}
Assume $Q\ll P$, $v\le B$, $F_Q$ continuous, and let $q:=q_Q(1-\alpha)$. Then
\[
\sqrt m\,\big(P_e(\tauhat(1-\alpha))-\alpha\big)\ \Rightarrow\ N\big(0,\ \alpha^2K^2+(1-2\alpha)\,\EE_Q[v\ind{S>q}]\big).
\]
In particular, if $\{S>q\}=A$ $Q$-a.s.\ for a set $A$ with $Q(A)=\alpha$ on which $v$ takes its largest values (for instance $S=\ind{A}(X)+U$ with $U\sim\mathrm{Unif}(0,1)$ independent of $X$), the asymptotic variance is $\sigma_\alpha^2$.
\end{proposition}

The proof (Appendix~\ref{app:lower}) is the quantile-crossing argument of Theorem~\ref{thm:upper} run at a moving quantile: since $F_Q$ is continuous, $F_Q(\tauhat(\beta))$ is the self-normalized weighted $\beta$-quantile of $W_i:=F_Q(S_i)$, which are uniform under $Q$, and $\{F_Q(\tauhat(\beta))\le\beta+z/\sqrt m\}$ is the event that a sum of i.i.d.\ bounded variables with mean $z/\sqrt m$ and variance converging to the displayed one is nonnegative.

\begin{remark}[The worst case and an exact benchmark]\label{rem:beta}
The worst case for the weighted quantile, $\sigma_\alpha^2=\alpha(1-\alpha)B$, is $v\in\{0,B\}$ $Q$-a.s., i.e.\ $Q_X$ concentrated on a set of $P_X$-mass $1/B$ (then $K^2=B$ and Theorem~\ref{thm:lower} gives the matching rate). There, the weighted quantile is the ordinary empirical quantile of the $N\sim\mathrm{Bin}(m,1/B)$ calibration points that carry weight, so for continuous scores and $N\ge1$,
\[
1-P_e(\tauhat(\beta))\ \big|\ N\ \sim\ \mathrm{Beta}\big(\lceil\beta N\rceil,\,N+1-\lceil\beta N\rceil\big)\qquad\text{when }\lceil\beta N\rceil\le N,
\]
and the set is all of $\cY$ otherwise, as it is when $N=0$. This exact law is checked numerically in Appendix~\ref{app:experiments}; it gives
\[
\sd\big(P_e(\tauhat(1-\alpha))\big)=\sqrt{B\alpha(1-\alpha)/m}\,(1+o(1)),
\]
and the smallest inflation achieving $\PP(P_e>\alpha)\le\delta$ is approximately $z_{1-\delta}\sqrt{B\alpha(1-\alpha)/m}$.
\end{remark}

\section{Clipping and estimated inputs}\label{sec:estimated}

Theorem~\ref{thm:upper} takes the likelihood ratio and its functionals as known. This section treats three distinct departures from that assumption, which are easily conflated. First, the \emph{ratio itself is estimated}, from the training covariates and an unlabeled sample from $Q_X$: Section~\ref{sec:point} gives an oracle inequality for a point estimate and Section~\ref{sec:interval} an exact guarantee for interval-valued estimates. Second, the ratio is known but is \emph{deliberately clipped} to reduce its range at the price of a computable bias: Section~\ref{sec:clip} shows that this trade is available to the weighted quantile and to rejection sampling alike, and uses it to compare the two. Third, the ratio is known but its \emph{population functionals} $K^2$, $\Lambda(\alpha)$, $\rho_\alpha$ and the clipping bias are \emph{estimated} from an unlabeled source sample: Section~\ref{sec:est} makes the certificate rigorous in that case. \citet{pournaderi2026training} handled the first problem by adding the random term $2\sum_i|\widehat V_i-V_i|/\widehat V$ inside the weighted DKW bound; the quantile argument gives a cleaner decomposition.

\subsection{Estimated ratios, point estimates: an oracle inequality}\label{sec:point}

Let $\widehat v\ge0$ be a measurable function independent of the calibration sample (all statements are conditional on the data used to build it), with $0<\EE_P\widehat v<\infty$ and $\widehat B:=\|\widehat v\|_\infty<\infty$. Define
\[
\widetilde v:=\frac{\widehat v}{\EE_P\widehat v},\qquad \widehat Q:=\widetilde v\cdot P,\qquad \widetilde B:=\frac{\widehat B}{\EE_P\widehat v},
\]
\[
\breve F(t):=\frac{\sum_{i}\widehat v(Z_i)\ind{S_i\le t}}{\sum_{j}\widehat v(Z_j)},\qquad \breve\tau(\beta):=\inf\{t:\ \breve F(t)\ge\beta\},
\]
and note that
\[
\frac{d\widehat Q}{dP}=\widetilde v\le\widetilde B,\qquad \EE_P\widetilde v=1,\qquad F_{\widehat Q}(t)=\EE_P\big[\widetilde v\,\ind{S\le t}\big],
\]
and that $\breve\tau$ is unchanged if $\widehat v$ is replaced by $\widetilde v$. Let
\begin{equation}\label{eq:Delta}
\Delta:=\sup_{t\in\R}\big|F_Q(t)-F_{\widehat Q}(t)\big| .
\end{equation}

\begin{theorem}\label{thm:estimated}
For every $\beta\in(0,1)$, $\alpha:=1-\beta$ and $\gamma>0$, with $P_e(\tau)=1-F_Q(\tau)$ the miscoverage under the true $Q$,
\begin{equation}\label{eq:oracle}
\PP\big(P_e(\breve\tau(\beta))>\alpha+\Delta+\gamma\big)\le\min\Big\{e^{-2m\gamma^2/\widetilde B^2},\ \exp\Big(-\frac{m\gamma^2}{2\widetilde B\alpha(1-\alpha)+\frac83\widetilde B\gamma}\Big)\Big\}.
\end{equation}
More generally, every statement of Theorem~\ref{thm:upper} and Corollary~\ref{cor:pac} holds for $\breve\tau(\beta)$ with $(Q,B)$ replaced by $(\widehat Q,\widetilde B)$, and the miscoverage under $Q$ differs from that under $\widehat Q$ by at most $\Delta$. Moreover,
\begin{equation}\label{eq:Delta_bounds}
\Delta\ \le\ \tfrac12\,\EE_P\big|\EE_P[v-\widetilde v\mid S]\big|\ \le\ \TV(Q,\widehat Q)=\tfrac12\|v-\widetilde v\|_{L^1(P)}\ \le\ \|v-\widehat v\|_{L^1(P)} .
\end{equation}
\end{theorem}

The statistical term in \eqref{eq:oracle} depends on $(P,\widehat v)$ only through $\widetilde B=\|\widehat v\|_\infty/\EE_P\widehat v$; the normalizer $\EE_P\widehat v$ is in general unknown, like the functionals of Remark~\ref{rem:known}, and in practice $\widehat v$ is normalized to unit empirical mean on an independent source sample, so that $\widetilde B\approx\max\widehat v$ up to the estimation of that mean (a lower confidence bound on $\EE_P\widehat v$, as in Section~\ref{sec:est}, makes it rigorous). The bias $\Delta$ is unknown, and \eqref{eq:oracle} is a PAC guarantee at level $\alpha+\Delta$. The marginal counterpart of the bias is known---\citet[Prop.~1]{lei2021conformal} bound the marginal coverage loss by $\frac12\EE_P|\widehat w-w|$, and \citet{barber2023conformal} bound the coverage gap of a fixed-weight quantile by a weighted total variation---so the content of Theorem~\ref{thm:estimated} is that the same $L^1$ quantity controls the whole distribution of the training-conditional coverage, together with the first inequality in \eqref{eq:Delta_bounds}: only the part of the weight error that is correlated with the score, $\EE_P[v-\widetilde v\mid S]$, matters. Errors in directions of $\cX$ that do not change the conditional law of the score are harmless, and Section~\ref{sec:exp_estimated} shows that $\Delta$ can be an order of magnitude below $\TV(Q,\widehat Q)$. Two further comments. First, since $(P,Q)$ and $(P,\widehat Q)$ generate identical data, no procedure that observes only $(D,\widehat v)$ can tell which of the two targets it is calibrating for; the guarantee of the true-ratio quantile therefore cannot in general be transferred to the estimated-ratio quantile without information that controls the discrepancy $\Delta$ (a quantitative efficiency statement would need a separate argument). Second, independence of $\widehat v$ from the calibration sample is free when $\widehat v$ is built from the training split and unlabeled target covariates; using the calibration covariates as well would require uniformity over the class of possible $\widehat v$, i.e.\ exactly the empirical-process machinery that Theorem~\ref{thm:upper} avoids.

\subsection{Estimated ratios, interval estimates: exact PAC guarantees}\label{sec:interval}

If instead one has $v_{\rm lo}\le v\le v_{\rm hi}$ (pointwise, or on an event of probability $1-\delta'$ over the data used to build them), the exact PAC guarantee can be retained by a conservative quantile, in the spirit of the interval-valued weights of \citet[Sec.~3.3]{park2022pac}, at an observable price.

\begin{proposition}\label{prop:interval}
Let $0\le v_{\rm lo}\le v\le v_{\rm hi}\le B_{\rm hi}$ pointwise, with $v_{\rm lo},v_{\rm hi}$ independent of the calibration sample. Let
\[
\begin{gathered}
A_{\rm lo}(t):=\sum_iv_{\rm lo}(Z_i)\ind{S_i\le t},\qquad C_{\rm hi}(t):=\sum_iv_{\rm hi}(Z_i)\ind{S_i>t},\\
\breve F(t):=\frac{A_{\rm lo}(t)}{A_{\rm lo}(t)+C_{\rm hi}(t)}
\end{gathered}
\]
(with $0/0:=0$) and $\breve\tau(\beta):=\inf\{t:\breve F(t)\ge\beta\}$. Then
\begin{enumerate}[label=(\roman*),leftmargin=2em,itemsep=0pt]
\item $\breve F$ is a nondecreasing right-continuous step function with $\breve F\le\Fhat$ pointwise, so $\breve\tau(\beta)\ge\tauhat(\beta)$, $P_e(\breve\tau(\beta))\le P_e(\tauhat(\beta))$, and Theorem~\ref{thm:upper} and Corollary~\ref{cor:pac}(i) hold for $\breve\tau(\beta)$ with $B$ replaced by $B_{\rm hi}$;
\item with the observable
\[
\widehat\rho:=\frac{\sum_i\big(v_{\rm hi}(Z_i)-v_{\rm lo}(Z_i)\big)}{\sum_iv_{\rm lo}(Z_i)}
\]
(and $\widehat\rho:=+\infty$ when the denominator vanishes), $\breve F\ge\Fhat-2\widehat\rho$ pointwise, hence $\breve\tau(\beta)\le\tauhat(\beta+2\widehat\rho)$; if $F_Q$ is continuous, $\alpha\le1/2$ and $\beta+2\rho_0+\gamma<1$,
\[
\PP\big(P_e(\breve\tau(\beta))<\alpha-2\rho_0-\gamma\big)\le\PP(\widehat\rho>\rho_0)+\exp\Big(-\frac{m\gamma^2}{2B_{\rm hi}\alpha(1-\alpha)+\frac83B_{\rm hi}\gamma}\Big),
\]
and, when $\EE_Pv_{\rm lo}>0$, $\widehat\rho$ concentrates around $\EE_P[v_{\rm hi}-v_{\rm lo}]/\EE_P v_{\rm lo}$, the relative $L^1(P)$ width of the intervals.
\end{enumerate}
\end{proposition}

The two subsections are complementary: a point estimate yields an exact statistical term and an unknown bias (possible under-coverage), an interval estimate yields an exact PAC guarantee and an observable over-coverage. Both prices are controlled from above by the $L^1(P)$ estimation error of the likelihood ratio, which is an upper bound and can be loose, since by \eqref{eq:Delta_bounds} only the score-projected error matters for $\Delta$. For a $(0.05,0.05)$-PAC guarantee $\Delta$ must be well below $\alpha$ and $\widetilde B$ must stay moderate; in Section~\ref{sec:exp_estimated} a well-specified parametric ratio model (logistic regression of ``target vs.\ source'' on the covariates, exact for exponential tilts) achieves this, whereas the kernel density-ratio estimate tested there, in five dimensions and at the sample sizes used, does not, mainly because it inflates $\widetilde B$.

\subsection{Clipping a known ratio and a certified frontier}\label{sec:clip}

Theorem~\ref{thm:estimated} has a second use: the ``estimate'' $\widehat v$ may be chosen deliberately. Clipping the likelihood ratio at a level $B'\in[1,B]$ trades a known bias for a smaller effective range, and the same trade is available to rejection sampling. Let
\[
\begin{gathered}
v_{B'}:=\min(v,B'),\qquad \widetilde v_{B'}:=\frac{v_{B'}}{\EE_Pv_{B'}},\qquad Q_{B'}:=\widetilde v_{B'}\cdot P,\\
\Delta_{B'}:=\TV(Q,Q_{B'})=\tfrac12\|v-\widetilde v_{B'}\|_{L^1(P)},
\end{gathered}
\]
a computable function of $v$ (through source covariates) with $\Delta_B=0$; note that $\Delta_1=\TV(Q,Q_1)$ with $\widetilde v_1=\min(v,1)/\EE_P\min(v,1)$, which differs from $\TV(Q,P)$ in general (see the second caution below).

\begin{proposition}[Clipped procedures]\label{prop:clip}
Fix $B'\in[1,B]$, $\alpha\in(0,1/2]$ and $\delta\in(0,1)$.
\begin{enumerate}[label=(\roman*),leftmargin=2em,itemsep=0pt]
\item (Clipped weighted quantile.) Let $\alpha':=\alpha-\Delta_{B'}>0$ and let $\gamma_{B'}$ be the inflation \eqref{eq:gamma_tail}, or the Bentkus inflation of Proposition~\ref{prop:bentkus}, computed for the target $\alpha'$ and the pair $(Q_{B'},\|\widetilde v_{B'}\|_\infty)$, i.e.\ with $K^2,\Lambda(\alpha'),\rho_{\alpha'}$ of $\widetilde v_{B'}$ and range
\[
\|\widetilde v_{B'}\|_\infty=\frac{\min(M,B')}{\mu_{B'}}\le\frac{B'}{\mu_{B'}},\qquad M:=\operatorname{ess\,sup}v,\quad \mu_{B'}:=\EE_Pv_{B'}
\]
(the larger bound $B'/\mu_{B'}$ is what is used when only $B\ge M$ is known). The weighted quantile with weights $v_{B'}(Z_i)$ at level $1-\alpha'+\gamma_{B'}=1-\alpha+\Delta_{B'}+\gamma_{B'}$ satisfies $\PP(P_e>\alpha)\le\delta$.
\item (Clipped rejection sampling.) Accepting each calibration point with probability equal to $\min(v_i/B',1)$ yields an i.i.d.\ sample from $Q_{B'}$; the Clopper--Pearson set of \citet{park2022pac} built on it at miscoverage level $\alpha-\Delta_{B'}$ satisfies $\PP(P_e>\alpha)\le\delta$.
\end{enumerate}
In both cases the certified over-coverage is a function of $B'$ alone---$\Delta_{B'}+\gamma_{B'}$ for (i), and $\Delta_{B'}+\big(\alpha-\Delta_{B'}-k^\star/\bar N\big)$ for (ii), where $\bar N=m\,\EE_P[v_{B'}]/B'$ is the expected accepted size and $k^\star$ the number of exceedances the Clopper--Pearson rule allows at that size---so the clip level can be chosen, before any label is seen, to minimize it.
\end{proposition}

Part (i) is Theorem~\ref{thm:estimated} with $\widehat v=v_{B'}$, $\widetilde B=B'/\mu_{B'}$ and $\Delta\le\Delta_{B'}$, applied at the target $\alpha'$ (note that the clipped range bound is $B'/\mu_{B'}>B'$ in general, since $\mu_{B'}\le1$, and that the functionals must be taken at $\alpha'$ rather than $\alpha$, because $\sigma_a^2$ is not monotone in $a$); part (ii) is the PAC guarantee of \citet{park2022pac} for the target $Q_{B'}$ together with $|F_Q-F_{Q_{B'}}|\le\Delta_{B'}$. Since $B'$ is chosen from population functionals of $v$ (in practice from an independent source sample) and not from the calibration data, the selection does not affect validity. Two cautions. First, the certified over-coverage is a quantile-level quantity used to \emph{select} $B'$; the guarantee is the validity statement, not the value of the over-coverage, and for rejection sampling the criterion is evaluated at the expected accepted size and is a selection proxy. Second, clipping at $B'=1$ does not in general recover the source distribution: $\widetilde v_1=\min(v,1)/\mu_1$ is constant only when $v\ge1$ a.s., and the acceptance probabilities $\min(v_i,1)$ are all one only in that case. In the rare-subgroup designs of Section~\ref{sec:exp_frontier} the bulk has $v=1-\pi\approx1$ and $B'=1$ is nearly, but not exactly, plain Clopper--Pearson calibration of the source scores; the latter, with the allowance $\TV(Q,P)$, is a separate deterministic candidate that we also report. Proposition~\ref{prop:clip} places the two constructions on a common footing and makes the comparison quantitative. Writing $L=\log(1/\delta)$, $z=z_{1-\delta}$ and $a:=\alpha-\Delta_{B'}$ for the effective target, and keeping only the leading terms in $1/\sqrt m$, the certified over-coverages of the clipped weighted quantile and of clipped rejection sampling are
\begin{equation}\label{eq:proxies}
\Delta_{B'}+\sqrt{\frac{2L\,\sigma_a^2(Q_{B'})}{m}}\qquad\text{and}\qquad \Delta_{B'}+z\sqrt{\frac{a(1-a)B'}{m\,\mu_{B'}}},
\end{equation}
where $\sigma_a^2(Q_{B'})$ is the variance proxy of $\widetilde v_{B'}$ at the target $a$. These are leading-order proxies for the certified over-coverage, not exact finite-sample comparisons of widths (the map from over-coverage to width depends on the score distribution, and the lower-order terms differ), and the statements below are to be read as such. Three consequences follow.
\begin{itemize}[leftmargin=1.5em,itemsep=1pt]
\item \emph{Unclipped, the weighted quantile has the smaller proxy iff}
\begin{equation}\label{eq:crossover}
\sigma_\alpha^2<\frac{z^2}{2L}\,\alpha(1-\alpha)B
\end{equation}
(for $\delta<1/2$, so that $z>0$). For small $\alpha$, where $\sigma_\alpha^2\approx\alpha\,\mathrm{CVaR}_\alpha(v)$, this reads $\mathrm{CVaR}_\alpha(v)\lesssim0.45B$ at $\delta=0.05$ ($0.59B$ at $\delta=0.01$, $0.69B$ at $\delta=10^{-3}$). The factor $z^2/2L$ is the price of a distribution-free concentration inequality against an exact binomial tail; it is what rejection sampling buys with its subsampling.
\item \emph{As $m\to\infty$ with the shift fixed}, any clip with $\Delta_{B'}>0$ is eventually inferior to one with $\Delta_{B'}=0$, so both optima approach the zero-bias set $\{B':\Delta_{B'}=0\}$ (an optimal clip may retain a positive bias that tends to zero). Let $M:=\operatorname{ess\,sup}v\le B$. When $v$ is nonconstant on $\{v>0\}$ the zero-bias set is $[M,B]$, since a clip below $M$ changes the normalized ratio on a set of positive probability; when $v=c\ind{A}$ it is all of $[1,B]$. Over the zero-bias set the weighted quantile is unchanged, whereas the effective range $B'/\mu_{B'}$ of rejection sampling is smallest at $B'=M$ (where $\mu_M=1$; for $v=c\ind A$, at any $B'\le c=M$), larger clips only discarding more observations. The optimized leading-order comparison is therefore the first item with $B$ replaced by $M$: the weighted quantile has the smaller proxy iff \eqref{eq:crossover} holds with $B$ replaced by $M$ ($\delta<1/2$; equality requires the lower-order terms). The distinction matters when $B$ is a loose bound: for $v\equiv1$ with $B=4$ and $\alpha=\delta=0.05$, the first item favours the weighted quantile against unclipped rejection sampling, but optimized rejection sampling clips at $B'=1$ and has leading proxy $0.358/\sqrt m$ against $0.533/\sqrt m$ for the Bernstein-inflated quantile.
\item \emph{At finite $m$}, clipping is cheap when the shift is small in total variation. The extreme case is to ignore the shift altogether: since
\[
|Q(S>\tau)-P(S>\tau)|\le\TV(Q,P)\qquad\text{for every }\tau,
\]
plain Clopper--Pearson calibration of the source scores at level $a_0:=\alpha-\TV(Q,P)$ is a deterministic PAC procedure when $0<\TV(Q,P)<\alpha$ (an alternative to the CP-C baseline of \citet{park2022pac}, which uses the level $\alpha/B$, and the less conservative of the two exactly when $\TV(Q,P)\le\alpha(1-1/B)$; it is void when $\TV(Q,P)\ge\alpha$ and is unweighted Clopper--Pearson PAC calibration when $\TV(Q,P)=0$). Against it, for $\delta<1/2$ and $\sqrt{2L\sigma_\alpha^2}>z\sqrt{a_0(1-a_0)}$, the unclipped weighted quantile has the smaller proxy iff $m>m^\star$, where
\begin{equation}\label{eq:mstar}
m^\star:=\bigg[\frac{\sqrt{2L\sigma_\alpha^2}-z\sqrt{a_0(1-a_0)}}{\TV(Q,P)}\bigg]^2.
\end{equation}
\end{itemize}
The weighted quantile is therefore the large-$m$, small-$\delta$ and heavy-tail end of the frontier, and clipped rejection sampling the small-$m$, mild-shift end. In the running example, clipping at $B'=1$ removes the subgroup's excess weight entirely: the bias is $\Delta_1=0.010$, the subgroup's excess test mass, so the effective target drops from $0.05$ to $0.04$ while the range drops from $21$ to $1.01$; at $B'=5$ the bias is still $0.008$ for a range of $5.0$. At $\alpha=\delta=0.05$ the certified criterion selects the extreme clip for rejection sampling at $m\le5000$ and no clip at $m=20000$, and levels $5$, $10$ and $21$ for the weighted quantile at $m=2000$, $5000$ and $20000$; Section~\ref{sec:exp_frontier} shows both ends of the frontier.

\subsection{Known ratio, estimated functionals: certificates from an unlabeled source sample}\label{sec:est}

The tail-adaptive and clipped certificates above depend on population functionals of $v$ under $P_X$ (the sup-$B$ Hoeffding and Bernstein certificates, CP-C and rejection sampling as published need only $B$): $K^2$, $\Lambda(\alpha)$, $\rho_\alpha$, and for the clipped versions $\mu_{B'}$ and $\Delta_{B'}$. Remark~\ref{rem:known} noted that they are typically estimated from an unlabeled source sample. Here we make the certificate rigorous in that situation: the functionals are replaced by simultaneous upper confidence bounds computed from one half of the source sample, the clip level and two auxiliary grid points are chosen on the other half, and the confidence level of the bounds is charged to $\delta$. The construction uses two facts: the bias of clipping is dominated by a single bounded expectation, and $\Lambda$ has a variational representation whose minimizer can be chosen on the independent half.

\begin{lemma}\label{lem:clipbias}
For $B'\in[1,B]$,
\[
\Delta_{B'}\ \le\ 1-\mu_{B'}\ =\ \EE_P(v-B')^+,
\]
and for $a\in(0,1]$ and any $s\ge0$,
\[
\Lambda(a)=\inf_{t\ge0}\big\{at+\EE_Q(v-t)^+\big\}\le a\,s+\EE_P\big[v(v-s)^+\big],
\]
with equality at $s=\rho_a$. The same holds for $Q_{B'}$ and $\widetilde v_{B'}$ in place of $Q$ and $v$.
\end{lemma}

Let $X'_1,\dots,X'_n\sim P_X$ be i.i.d.\ and independent of the calibration data, with $v$ known and $v\le B$ for a known $B$, as in \citet{park2022pac}, and split the index set into $I_1$ and $I_2$ of sizes $n_1$ and $n_2$. For a function $g$ with values in $[0,R]$ and a level $\eta\in(0,1)$, let $U_\eta(g)$ denote an upper confidence bound for $\EE_Pg(v)$ computed from $\{v(X'_j)\}_{j\in I_2}$, i.e.\ a statistic with
\[
\PP\big(\EE_Pg(v)\le U_\eta(g)\big)\ge1-\eta;
\]
Hoeffding's bound $\bar g+R\sqrt{\log(1/\eta)/(2n_2)}$ is one choice, and the empirical Bernstein bound of \citet[Thm.~4]{maurer2009empirical},
\begin{equation}\label{eq:eb}
U_\eta(g)=\bar g+\sqrt{\frac{2\widehat{\Var}(g)\log(2/\eta)}{n_2}}+\frac{7R\log(2/\eta)}{3(n_2-1)},
\end{equation}
with $\widehat{\Var}$ the sample variance with denominator $n_2-1$, is much better when $g(v)$ is large only on a rare event, which is the case of interest.

\begin{proposition}[Certificate with estimated functionals]\label{prop:est}
Fix $\alpha\in(0,1/2]$, $0<\delta'<\delta<1$, and let $B'\in[1,B]$ and two grid points $s_L,s_R\in[0,B']$ be arbitrary functions of $\{v(X'_j)\}_{j\in I_1}$ (and of $m$, $\alpha$, $\delta$); $s_L$ enters the bound on $\Lambda$ and $s_R$ the test that certifies $\rho$. Let $J=3$ for the Bernstein inflation \eqref{eq:gamma_tail}, $J=4$ for the Bentkus inflation, $\eta:=\delta'/J$, and compute from $I_2$, with $v_{B'}=\min(v,B')$ and with $r_+:=0$ when $B'=B$ (the function $(v-B')^+$ has range $B-B'$), the quantities
\begin{align*}
r_+&:=U_\eta\big((v-B')^+\big),\qquad \mu_-:=1-r_+,\qquad \alpha_+:=\alpha-r_+,\qquad \widetilde B_+:=B'/\mu_-,\\
K^2_+&:=\min\big\{U_\eta(v_{B'}^2)/\mu_-^2,\ \widetilde B_+\big\},\\
\Lambda_+&:=\min\big\{\alpha_+s_L/\mu_-+U_\eta\big(v_{B'}(v_{B'}-s_L)^+\big)/\mu_-^2,\ \alpha_+\widetilde B_+,\ K^2_+\big\},\\
\rho_+&:=s_R/\mu_-\ \text{ if }\ U_\eta\big(v_{B'}\ind{v_{B'}>s_R}\big)\le\alpha_+\mu_-,\ \text{ and }\ \rho_+:=\widetilde B_+\ \text{ otherwise (Bentkus only)}.
\end{align*}
Suppose $n_2\ge2$ and $m>\frac23\log(1/(\delta-\delta'))$ (the domain of \eqref{eq:gamma_tail}); if $\alpha_+\le0$ the remaining quantities are not needed and the output below is $\cY$. Otherwise let $\widehat\gamma$ be the inflation \eqref{eq:gamma_tail}, or that of Proposition~\ref{prop:bentkus}, evaluated with
\[
(\alpha,K^2,\Lambda(\alpha),\rho_\alpha,B,\delta)\quad\text{replaced by}\quad(\alpha_+,K^2_+,\Lambda_+,\rho_+,\widetilde B_+,\delta-\delta').
\] Let $\breve\tau$ be the weighted quantile with weights $v_{B'}(Z_i)$ at level $1-\alpha+r_++\widehat\gamma$ if $\alpha_+>0$, $\widehat\gamma$ is finite and nonnegative and $\widehat\gamma+r_+<\alpha$, and $\breve\tau:=+\infty$ (the set is all of $\cY$) otherwise, including when $m\le\frac23\log(1/(\delta-\delta'))$. Then $\PP(P_e(\breve\tau)>\alpha)\le\delta$, the probability being over the source sample and the calibration data jointly. Likewise, with $J=1$, the clipped rejection-sampling set of Proposition~\ref{prop:clip}(ii) at miscoverage level $\alpha_+$ and confidence $1-(\delta-\delta')$ (all of $\cY$ if $\alpha_+\le0$ or no threshold qualifies) satisfies $\PP(P_e>\alpha)\le\delta$.
\end{proposition}

The proof is in Appendix~\ref{app:estimated}. Three points about the construction. \emph{Independence and confidence allocation.} The clip level $B'$ and the grid points $s_L,s_R$ are chosen on $I_1$, so each confidence bound on $I_2$ is applied to a fixed function conditionally on $I_1$; a plain union bound over the $J\le4$ bounds charges $\delta'$ once, and the inflation is computed at $\delta-\delta'$. $B'$ can be selected as in Proposition~\ref{prop:clip}, by minimizing the certified over-coverage $r_++\widehat\gamma$ recomputed with $I_1$ in the role of $I_2$, and $s_L$ is the empirical $(1-\alpha_+)$-quantile of $v_{B'}$ under the clipped test law on $I_1$, the minimizer in Lemma~\ref{lem:clipbias}; a suboptimal $s_L$ only loosens $\Lambda_+$. \emph{Fallbacks.} The output is $\cY$ when $\alpha_+\le0$, when $m$ is below the domain of \eqref{eq:gamma_tail}, or when the certified over-coverage $r_++\widehat\gamma$ reaches $\alpha$; and the certificate of $\rho_{\alpha_+}$ is a one-sided test whose failure returns the range $\widetilde B_+$, which is valid but can be far from $\rho_{\alpha_+}$. \emph{The margin rule for $s_R$.} Where that test is placed matters. At the empirical quantile $s_L$ the estimated tail mass is about $\alpha_+\mu_-$ and the confidence bound exceeds it by its half-width, so the test fails with a probability that does not vanish as $n_2$ grows; at the smallest $I_1$-value where the $I_1$-version of the test passes, the $I_2$-bound differs from the $I_1$-bound by a symmetric fluctuation of order $n^{-1/2}$, and the failure probability tends to one half. We therefore select with a margin that dominates this fluctuation and still vanishes. Writing $U^{(1)}_\eta$, $\alpha^{(1)}_+$, $\mu^{(1)}_-$ for the quantities of Proposition~\ref{prop:est} computed on $I_1$ and $\widehat{\Var}_1$ for the sample variance on $I_1$, we take
\begin{equation}\label{eq:margin}
s_R:=\min\Big\{s\in\{v_{B'}(X'_j)\}_{j\in I_1}:\ U^{(1)}_\eta\big(v_{B'}\ind{v_{B'}>s}\big)+\kappa_1(s)\le\alpha^{(1)}_+\mu^{(1)}_-\Big\},
\end{equation}
with
\[
\kappa_1(s):=\sqrt{\frac{2\,\widehat{\Var}_1\big(v_{B'}\ind{v_{B'}>s}\big)\log n_1}{n_1}},
\]
and $s_R:=B'$ if the set is empty; $\kappa_1$ is the variance term of the empirical Bernstein half-width at confidence level $2/n_1$. Proposition~\ref{prop:margin} in Appendix~\ref{app:margin} shows that, with the empirical Bernstein bounds and under mild regularity of the law of $v_{B'}$ (no atom at the target quantile, positive mass around it), the test on $I_2$ then passes with probability tending to one and $\rho_+$ converges to the quantile the certificate targets; the appendix also treats atoms and gives an example with $B=100$ in which the fallback would cost ten percent of the inflation.

Two further remarks. First, rejection sampling as published needs no functional beyond $B$ and pays nothing here; the price of the weighted quantile's sharper certificate is that its inputs must be estimated, and Section~\ref{sec:exp_est} quantifies that price. Second, the estimation error enters in two places with different sensitivities: the bias bound $r_+$ is added to the over-coverage directly, so its half-width, the last two terms of \eqref{eq:eb} with range $R=B-B'$, must be small compared with $\widehat\gamma\approx\sqrt{2L\sigma_\alpha^2/m}$, whereas the bounds on $K^2$ and $\Lambda$ enter under a square root, with range $B'^2$. For a bounded quantity with range $R_g$ the range term dominates when $n_2\Var(g)\lesssim R_g^2\log(1/\eta)$; for the tail functionals of the running example, whose variance is driven by the few subgroup points, this is the regime of a small source sample: the certification half of a $10^4$-point sample contains on average $2.5$ subgroup points. Clipping has two opposite effects: a lower $B'$ reduces the range $B'^2$ of the moment estimates, but the bias bound $r_+$, with range $B-B'$, grows and enters the over-coverage directly, which can favour $B'=B$ when the source sample is small. For the unclipped certificate, the source sample needed for the estimation error in $K^2$ and $\Lambda$ to be negligible is of order $B^2\log(1/\eta)/\sigma_\alpha^2$.

\section{Experiments}\label{sec:experiments}

Code reproducing every figure and table is available in the repository listed in the Supplementary Material. Throughout, ``fail'' denotes the fraction of calibration draws with $P_e(D)>\alpha$, which the PAC guarantee bounds by $\delta$; ``none'' is the weighted quantile at level $1-\alpha$; ``Hoeff.''\ and ``Bern.\ $B$'' the weighted quantile at level $1-\alpha+\gamma_m(\delta)$ with the Hoeffding and the sup-$B$ Bernstein inflation of \eqref{eq:gamma_m}; ``Bern.\ tail'' the weighted quantile with the tail-adaptive inflation \eqref{eq:gamma_tail}, where $K^2$, $\Lambda(\alpha)$ and $B$ are the population functionals of $v$ (exact for the synthetic shifts of Sections~\ref{sec:exp_smooth} and \ref{sec:exp_frontier}--\ref{sec:exp_est}, Appendix~\ref{app:experiments}, and computed on the pool for real data); ``WQ-Bk'' the same with the Bentkus inflation of Proposition~\ref{prop:bentkus}; ``RS+CP'' the rejection-sampling Clopper--Pearson set of \citet{park2022pac} (accept each calibration point with probability $v_i/B$, then choose the smallest threshold whose Clopper--Pearson upper confidence bound on the miscoverage, at confidence $1-\delta$, is at most $\alpha$); ``IW-HB'' and ``IW-EB'' the importance-weighted learn-then-test procedure of \citet{almeida2025high}, which tests, for each threshold $\lambda$ of a fixed grid in decreasing order, the null hypothesis $\EE_P[v\ind{S>\lambda}]>\alpha$ with the importance-weighted $0$--$1$ loss $v_i\ind{S_i>\lambda}\in[0,B]$ and stops at the first non-rejection, using either the Hoeffding--Bentkus p-value of \citet{bates2021distribution} on the loss rescaled to $[0,1]$ or the empirical Bernstein rule
\[
\widehat R+\widehat\sigma\sqrt{2\log(2/\delta)/m}+\frac{7B\log(2/\delta)}{3(m-1)}\le\alpha
\]
(both in their Appendix~B; their preferred betting p-value depends on the order of the data and is not reproduced here); and ``empirical oracle'' the weighted quantile with the smallest inflation that achieves $\PP(P_e>\alpha)\le\delta$ on the simulated draws themselves, found by bisection; being tuned on the draws on which it is evaluated, it is an in-sample benchmark for the unclipped weighted-quantile inflation family, not a procedure and not a lower bound for other certificates; its gap from the certificates reflects both the replacement of the actual score--weight coupling by $\Lambda(\alpha)$ and the looseness of the concentration inequality. We use $\delta=0.05$ except in Table~\ref{tab:frontier}, which also reports $\delta=0.01$.

\subsection{Smooth covariate shift with known weights}\label{sec:exp_smooth}

Next, $X\sim\mathrm{Unif}[-1,1]^5$ under $P$, $Q_X$ has density proportional to $e^{cx_1}$ on the cube, so that $v(x)=e^{cx_1}/\EE_Pe^{cX_1}$ and $B=e^{c}/\EE_Pe^{cX_1}$; \[
Y=1+X_1+0.5X_2-0.5X_3^2+\big(0.5+(X_1+1)\big)\varepsilon,\qquad \varepsilon\sim N(0,1),
\]
so that the noise level, and hence the score distribution, depends on the shifted coordinate; $\widehat\mu$ is the least-squares fit on an independent training sample of size $1000$ and $s(x,y)=|y-\widehat\mu(x)|$. Training-conditional coverage is evaluated on $3\times10^5$ test points from $Q$, and the population functionals of $v$ are exact (Appendix~\ref{app:experiments}). Table~\ref{tab:smooth} reports, for $m=2000$, $\alpha=\delta=0.05$ and $c\in\{0,1,2,3\}$, the procedures that are compared on the frontier in Section~\ref{sec:exp_frontier}: the tail-adaptive Bernstein inflation \eqref{eq:gamma_tail} (WQ), the Bentkus inflation of Proposition~\ref{prop:bentkus} (WQ-Bk) and its clipped version at the certified clip level of Section~\ref{sec:clip}, rejection sampling with a Clopper--Pearson bound (RS+CP) and its clipped version, and the two concentration-based importance-weighted learn-then-test procedures of \citet{almeida2025high} (IW-HB, IW-EB; Section~\ref{sec:exp_frontier} describes them), with the uninflated weighted quantile as the reference.

\begin{table}[tb]
\centering\small
\caption{Smooth shift, $m=2000$, $\alpha=\delta=0.05$, $1000$ replications, exact population functionals. Widths are medians of $2\tauhat$ relative to the uninflated set; clip levels in parentheses. WQ: tail-adaptive Bernstein inflation; WQ-Bk: Bentkus inflation; RS+CP: rejection sampling with Clopper--Pearson; IW-HB, IW-EB: importance-weighted learn-then-test with Hoeffding--Bentkus and empirical Bernstein p-values (Section~\ref{sec:exp_frontier}).}
\label{tab:smooth}
{\setlength{\tabcolsep}{4pt}\input{tables/tab_smooth.tex}}
\end{table}

For an exponential tilt the supremum is attained on a set of large test probability, so $\mathrm{CVaR}_\alpha(v)$ is close to $B$ ($3.97$ vs.\ $4.07$ for $c=2$) and this is essentially a worst-case setting for the weighted quantile's bound, although not for the weighted quantile itself: $K^2=1+\chi^2(Q\|P)$ is only $2.1$ for $B=4.07$, and the oracle inflation is accordingly small ($5.8\%$ extra width). The uninflated quantile fails for about half of the calibration sets, and every certified procedure meets the target. The Hoeffding inflation (not shown) is infeasible as soon as $B>1$, and the sup-$B$ Bernstein inflation \eqref{eq:gamma_m} gives $10$--$41\%$ extra width, against $8$--$20\%$ for the tail-adaptive inflation, whose gain comes from the lower-order terms since $\Lambda(\alpha)\approx\alpha B$ here, and $7$--$16\%$ for the Bentkus inflation. Clipping buys nothing on this shift: the certified clip level is $B$ itself for the weighted quantile in every row and for rejection sampling in all but the last, where a level of $5.6$ is certified marginally better but not narrower in the realized sets, as the crossover condition of Section~\ref{sec:clip} predicts when $\mathrm{CVaR}_\alpha/B\approx1$. Rejection sampling is the narrowest procedure in this regime, $5$--$12\%$ wider than the uninflated set and $2$--$5$ width points narrower than WQ-Bk, and it uses most of its $\delta$ budget (it fails for $3$--$5\%$ of the calibration sets); its coverage is as variable across calibration sets as that of the uninflated quantile (standard deviation $0.0045$--$0.0092$ against $0.0050$--$0.0096$ as $B$ grows from $1$ to $6$), whereas the inflated weighted quantiles are less variable ($0.004$--$0.007$); its variability across reruns on a fixed calibration set is measured in Section~\ref{sec:exp_frontier}. The importance-weighted learn-then-test set with the Hoeffding--Bentkus p-value is within a width point of WQ-Bk ($1.070$--$1.171$ against $1.071$--$1.164$), and the one with the empirical Bernstein p-value is $11$--$32\%$ wider than the uninflated set. This mild-shift regime is the one where the exact binomial tail of rejection sampling is expected to win (Section~\ref{sec:clip}); the rare-event designs of Section~\ref{sec:exp_frontier} reverse the order.

\subsection{Estimated likelihood ratios}\label{sec:exp_estimated}

In the design of Section~\ref{sec:exp_smooth} with $c=2$ ($B=4.07$) and $\alpha=0.1$, we estimate $v$ from $n_w$ source and $n_w$ target covariates, independent of the calibration sample, by (a) logistic regression of the sample indicator on $x$, which is well specified here, and (b) a Gaussian kernel density ratio; the estimate is normalized to unit mean on the source covariates and $\widetilde B$ is the supremum of the normalized estimate, analytic for (a) and, for (b), the empirical maximum over $5\times10^4$ source draws, which is not a certified supremum (Remark~\ref{rem:known}). Table~\ref{tab:estimated} reports the $L^1(P)$ error, $\TV(Q,\widehat Q)$, the discrepancy $\Delta$ of \eqref{eq:Delta}, and the failure fractions of the uninflated and the Bernstein-inflated weighted quantile with estimated weights, all evaluated under the true $Q$. With the well-specified logistic estimate, the $L^1$ error of $\widehat v$ decreases from $0.12$ ($n_w=250$ and $1000$) to $0.04$ ($n_w=4000$) and $\TV(Q,\widehat Q)$ from $0.058$ to $0.020$, while the discrepancy $\Delta$ that actually enters the bound is $0.002$--$0.005$ in all three cases, at the resolution of its own Monte Carlo evaluation, i.e.\ $4$--$28$ times smaller than $\TV(Q,\widehat Q)$ and $9$--$56$ times smaller than the $L^1$ error. At $n_w\le1000$ the estimate also overstates the range ($\widetilde B=6.2$--$6.5$ against the true $B=4.07$), which inflates $\gamma_m(\delta)$ from $0.045$ to $0.055$--$0.057$; this is the other cost of a poor estimate. Since $\Delta$ is so small, the uninflated procedure behaves as with the true weights (it fails for $45$--$57\%$ of the calibration sets, with median coverage $0.90$), and the Bernstein inflation never fails at the nominal level $\alpha$ in $500$ replications because $\gamma_m(\delta)\gg\Delta$; had $\Delta$ been larger one could have inflated by it (unknown in practice) or used interval weights (Proposition~\ref{prop:interval}). The kernel density ratio in five dimensions behaves differently: its $L^1$ error is $0.69$ and $\TV(Q,\widehat Q)=0.19$, yet $\Delta=0.003$, because most of its error lies in covariate directions that do not affect the score distribution; the estimate is unusable for a PAC guarantee nonetheless, because its (empirical) supremum $\widetilde B=36$ makes the inflation infeasible ($\gamma_m(\delta)=0.20>\alpha$). This is the practical content of Theorem~\ref{thm:estimated}: a ratio estimate has to be accurate on the score scale \emph{and} bounded to be useful for training-conditional guarantees.
\begin{table}[tb]
\centering\small
\caption{Estimated weights (smooth shift, $c=2$, $m=2000$, $\alpha=0.1$, $\delta=0.05$, $500$ replications). The Bernstein inflation uses $\widetilde B$ from the estimate; failures are counted at the nominal $\alpha$ (not at $\alpha+\Delta$).}
\label{tab:estimated}
\resizebox{\linewidth}{!}{\input{tables/tab_estimated.tex}}
\end{table}

\subsection{Real data}\label{sec:exp_real}

Finally we revisit the four UCI regression datasets of \citet{pournaderi2026training} (Wine quality, red and white combined; Abalone; Concrete compressive strength; Combined cycle power plant), with the same exponential tilts $\exp(x^\top\beta_{\rm data})$ to define the shifted test population, so that $v$ is known up to a normalization, computed on the pool defined next. For each dataset a random forest is fitted once on a training split; the remaining $n_{\rm pool}$ points form a pool whose empirical distribution plays the role of $P$ and whose $v$-tilt plays the role of $Q$, so that calibration sets of size $m$ drawn with replacement from the pool are i.i.d.\ from $P$, the training-conditional coverage is the $v$-weighted coverage on the pool, and the functionals of $v$ are exact. Since $\gamma_m(\delta)<\alpha$ requires $m$ of order $B\log(1/\delta)/\alpha$, we use $\alpha=0.1$ and larger calibration sets than in \citet{pournaderi2026training}. Table~\ref{tab:real} reports the procedures of Table~\ref{tab:smooth}, once with the original tilts and once with the tilts rescaled so that $B=4$.

\begin{table}[tb]
\centering\small
\caption{UCI datasets with exponential-tilt covariate shift, $\alpha=0.1$, $\delta=0.05$, $1000$ calibration draws from the pool; procedures and conventions as in Table~\ref{tab:smooth}; sample sizes in Appendix~\ref{app:experiments}. ``$\infty$'' means that the set is all of $\cY$ for more than half of the draws.}
\label{tab:real}
{\footnotesize\setlength{\tabcolsep}{3.5pt}\input{tables/tab_real.tex}}
\end{table}

The question here is which certificates are feasible on real shifts and at what width. With the original tilts the ratio bounds are large ($B$ between $7$ and $58$) but, since the tilts act on unbounded features, the supremum is attained on a small part of the test population ($\mathrm{CVaR}_\alpha/B$ between $0.47$ and $0.89$). The sup-$B$ Bernstein inflation (not shown) is infeasible on three of the four datasets, whereas the tail-adaptive certificates are feasible on all four; the Bentkus inflation is $2$--$44$ width points narrower than the Bernstein one, the largest gain on Concrete ($m=500$), where the Bernstein inflation is close to $\alpha$. Among all valid sets, clipped rejection sampling is the narrowest on every original tilt, by $2$--$24$ width points over the clipped Bentkus-inflated quantile, in line with the crossover condition of Section~\ref{sec:clip} for $\mathrm{CVaR}_\alpha/B$ of about $0.5$ or more and calibration sets of a few thousand; the learn-then-test set with the Hoeffding--Bentkus p-value is wider than WQ-Bk on three datasets and equal on the fourth, and the empirical Bernstein variant is infeasible on three. Rescaling the tilts to $B=4$ brings $\mathrm{CVaR}_\alpha$ to $2$--$3.5$ and the weighted-quantile and rejection-sampling variants within a few width points of each other: WQ-Bk is within one width point of unclipped rejection sampling on all four datasets (narrower on Wine quality) and within two points of clipped rejection sampling on the three larger ones; the exception is Concrete, the smallest calibration set, where clipped rejection sampling is $9$ points narrower. Clipping selects $B'=B$ for the weighted quantile on every $B=4$ tilt and $B'\approx3$ for rejection sampling. Unclipped rejection sampling fails for $2.4$--$4.3\%$ of the calibration draws in every row, close to its nominal $\delta=5\%$; the inflated weighted quantiles fail for at most $0.3\%$, the usual slack of a distribution-free bound.

\subsection{The frontier: widths of valid sets when the supremum is a rare-event quantity}\label{sec:exp_rare}\label{sec:exp_frontier}

The settings above have $\mathrm{CVaR}_\alpha(v)\approx B$ or $B/2$. We now take shifts for which the supremum is attained on a small part of the test population, and compare all valid procedures by the cost at which they certify $\PP(P_e\le\alpha)\ge1-\delta$: the median width of the set relative to the uninflated weighted quantile (validity is checked in every cell). Two designs, both with the covariate distribution, least-squares predictor and absolute-residual score of Section~\ref{sec:exp_smooth}. \emph{Rare over-represented subgroup} (the running example of the introduction): $Q_X=(1-\pi)P_X+\pi P_X(\cdot\mid A)$ with $A=\{x_1>1-2p_A\}$, $P_X(A)=p_A$, so $v=(1-\pi)+\pi\ind A/p_A$; a subgroup that is a fraction $p_A$ of the training population is a fraction $\approx\pi<\alpha$ of the test population; the noise level is largest on $A$, so the score tail is positively associated with the high-ratio region; $(\pi,p_A)\in\{(0.01,5\cdot10^{-4}),(0.02,10^{-3}),(0.01,2\cdot10^{-4})\}$, i.e.\ $B\in\{21,21,51\}$, with $\TV(Q,P)\approx\pi$. \emph{Gaussian mean shift:} $x_1\sim N(0,1)$ truncated to $[-4,4]$ under $P$ and tilted by $e^{x_1}$ under $Q$, a heavy-tailed ratio with $B=33$, $\mathrm{CVaR}_\alpha=13$, $K^2=2.6$ and $\TV(Q,P)=0.38$; since $x_1$ ranges over $[-4,4]$ here rather than $[-1,1]$, the response model is
\[
Y=1+0.5X_1+0.5X_2-0.5X_3^2+\big(0.5+0.3|X_1|\big)\varepsilon,
\]
so that the noise level increases with $|X_1|$. Table~\ref{tab:frontier} reports, for $\delta\in\{0.05,0.01\}$ and $m\in\{2000,5000,20000\}$, the unclipped tail-adaptive weighted quantile with the Bernstein and with the Bentkus inflation, rejection sampling as published, the certified-optimally clipped versions of all three (Proposition~\ref{prop:clip}, clip level chosen separately for each by minimizing its certified over-coverage on a geometric grid, from the exact functionals), the two deterministic source-calibration baselines CP-C and CP+TV, the two learn-then-test variants, and the empirical oracle; Figure~\ref{fig:frontier} plots the competitive procedures against $m$.

\begin{figure}[tb]
\centering\includegraphics[width=\linewidth]{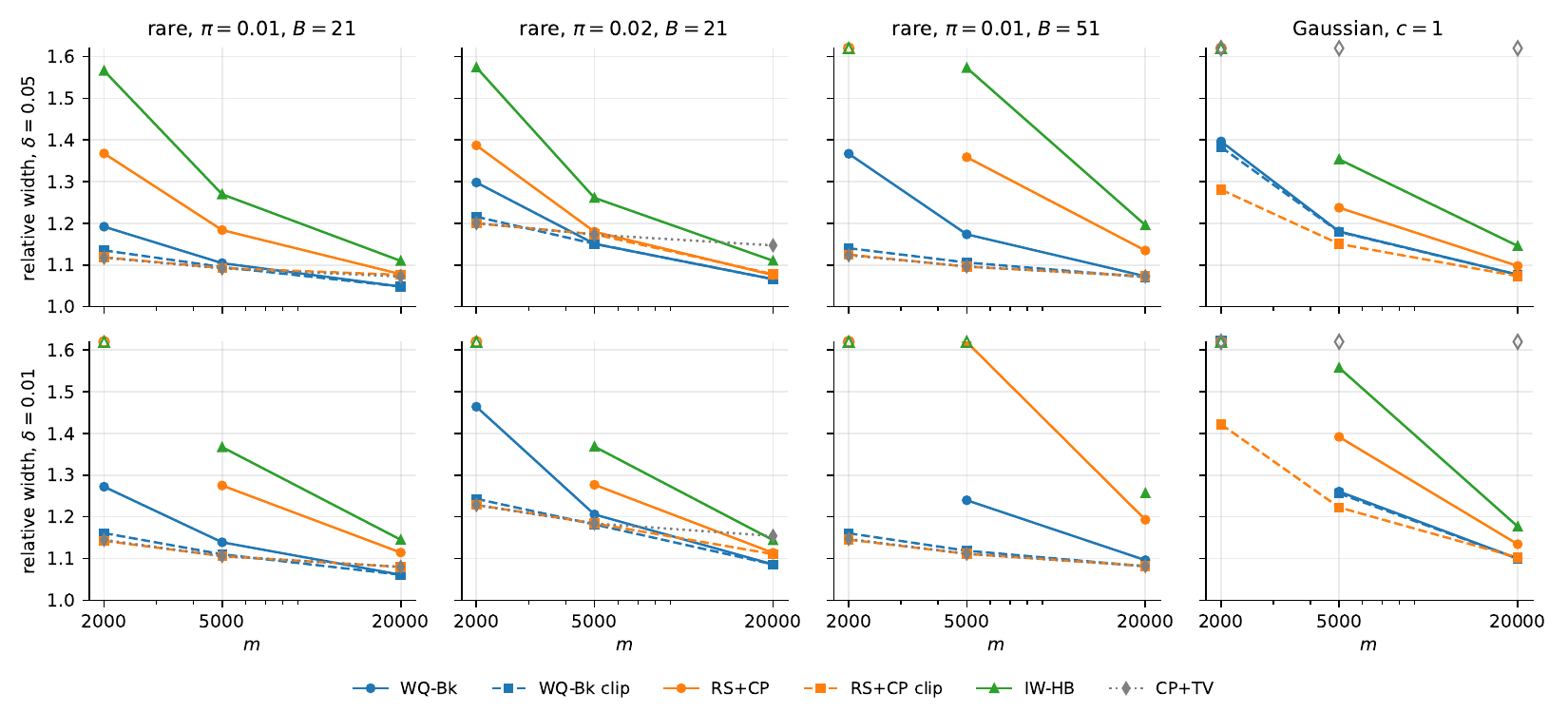}
\caption{The frontier of Table~\ref{tab:frontier}: median width relative to the uninflated weighted quantile against the calibration size $m$, for $\alpha=0.05$ and $\delta\in\{0.05,0.01\}$ (rows) in the four rare-event designs (columns). Solid lines are the unclipped procedures, dashed lines their certified-optimally clipped versions; hollow markers at the top of a panel mark cells in which the procedure returns $\cY$ for more than half of the draws or has relative width above $1.6$. Clipped rejection sampling leads at $m=2000$ in every panel; the clipped weighted quantile overtakes it between $m=2000$ and $5000$ in the $\pi=0.02$ design, between $5000$ and $20000$ in the other rare-subgroup designs and in the Gaussian shift at $\delta=0.01$ (by a few thousandths there), and not by $m=20000$ in the Gaussian shift at $\delta=0.05$, where clipped rejection sampling remains narrower ($1.073$ against $1.077$). The learn-then-test set with the Hoeffding--Bentkus p-value is wider than WQ-Bk throughout.}
\label{fig:frontier}
\end{figure}
\begin{table}[tb]
\centering\small
\caption{Widths of PAC sets relative to the uninflated weighted quantile ($\alpha=0.05$, $600$ replications; every entry has empirical failure fraction $\le\delta$ up to Monte Carlo error, standard error $\approx0.01$; widths are medians over all draws with the set $\cY$ counted as infinite width, ``$\infty$'' when more than half of the draws give $\cY$, and a superscript gives the percentage of draws giving $\cY$ when it is positive). ``WQ'' is the tail-adaptive weighted quantile with the Bernstein inflation and ``WQ-Bk'' with the Bentkus inflation, ``RS'' rejection sampling with Clopper--Pearson; ``clip'' the certified-optimally clipped version of each with its clip level $B'$ in parentheses; ``CP-C'' the source Clopper--Pearson baseline of \citet{park2022pac} at level $\alpha/B$, ``CP+TV'' source Clopper--Pearson at level $\alpha-\TV(Q,P)$, ``IW-HB'' and ``IW-EB'' importance-weighted learn-then-test \citep{almeida2025high} with the Hoeffding--Bentkus and the empirical Bernstein p-value. The narrowest certificate in each row is in bold; the empirical oracle is not a certificate.}
\label{tab:frontier}
\resizebox{\linewidth}{!}{\input{tables/tab_frontier.tex}}
\end{table}

The question is where the crossover of Section~\ref{sec:clip} falls when the supremum is a rare-event quantity, and the answer is visible in the figure: at $m=2000$ the clipped procedures are narrowest and clipped rejection sampling leads the clipped weighted quantile by $1$--$2$ width points in the rare-subgroup designs and by $10$--$21$ points in the Gaussian shift; at $m=5000$ the two are within a width point in the rare-subgroup designs (the weighted quantile ahead at $\pi=0.02$) and rejection sampling is $3$ points ahead in the Gaussian shift; at $m=20000$ the Bentkus-inflated weighted quantile, clipped where the certified level is below $B$, is the narrowest certificate in seven of the eight rows and within half a width point in the eighth, the certified clip levels having moved to $B'=B$ in five of them. The crossover sizes $m^\star$ of Section~\ref{sec:clip} for the three rare-subgroup designs at $\delta=0.05$ are $7{,}500$, $4{,}300$ and $20{,}200$ with the Bernstein constant and about a quarter smaller with the Bentkus constant, consistent with these outcomes. Three further findings. \emph{Unclipped}, the Bentkus-inflated weighted quantile is narrower than rejection sampling as published in every one of the $21$ cells in which the latter is feasible ($1$--$26\%$) and feasible in the three cells where it returns $\cY$; this is the regime of Remark~\ref{rem:scale}, where rejection sampling keeps $m/B$ points while the weighted quantile's certificate scales with $\mathrm{CVaR}_\alpha(v)\le0.45B$. \emph{The other deterministic procedures} are wider: CP-C certifies sets $1.7$--$2.3$ times the uninflated width in every cell; the learn-then-test set with the Hoeffding--Bentkus p-value is wider than rejection sampling wherever both are feasible and $4$--$44\%$ wider than WQ-Bk (equivalently, WQ-Bk is $4$--$30\%$ narrower), and the empirical Bernstein variant is infeasible in $14$ of the $24$ cells, its lower-order penalty $7B\log(2/\delta)/(3m)$ alone exceeding $\alpha/2$ at $m=5000$ for $B=21$; both penalties vanish as $m\to\infty$, so this is a finite-sample comparison, and the betting p-value that \citet{almeida2025high} prefer was not evaluated. \emph{The exception} to the weighted quantile's advantage at moderate $m$ is CP+TV, plain source calibration with a total-variation allowance, which is within a few thousandths of clipped rejection sampling at $B'=1$ in the rare-subgroup designs (the two are not identical, Section~\ref{sec:clip}) and is the narrowest entry in five of the twelve rare-subgroup cells with $m\le5000$; its bias floor $\TV(Q,P)$ shows at $m=20000$ ($1.147$ against $1.066$ at $\pi=0.02$), and it is void in the Gaussian shift, where $\TV(Q,P)=0.38>\alpha$. The empirical oracle gives sets only $2$--$17\%$ wider than the uninflated one, $1.7$--$3.5$ times less excess width than the best certificate; it is tuned on the draws on which it is evaluated and is a benchmark, not a procedure. Bootstrap standard errors of the median widths are below $0.021$ for $m\ge5000$ and below $0.07$ at $m=2000$, so differences of a few thousandths are not significant; failure fractions have a standard error of about $0.01$.

\begin{figure}[!tb]
\centering\includegraphics[width=\linewidth]{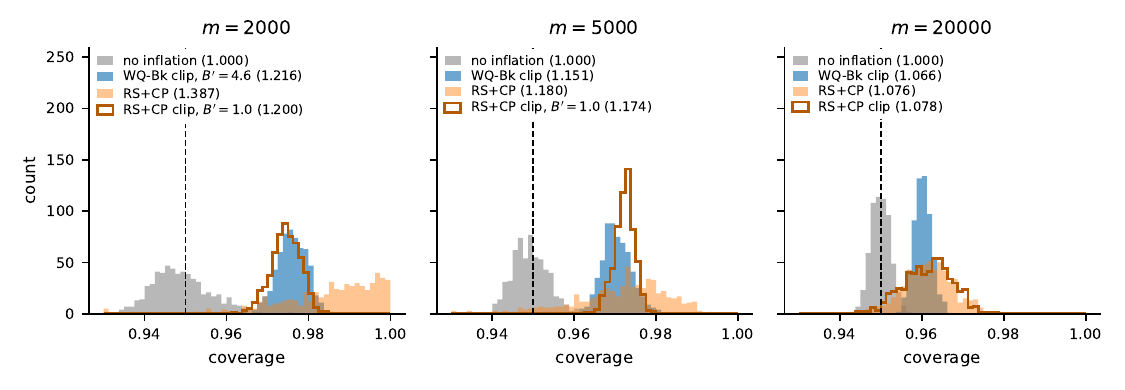}
\caption{Histograms of the training-conditional coverage over the $600$ calibration sets of Table~\ref{tab:frontier} in the rare-subgroup design with $\pi=0.02$, $B=21$ ($\alpha=\delta=0.05$): the uninflated weighted quantile, the Bentkus-inflated weighted quantile at its certified clip level (WQ-Bk clip), and rejection sampling with Clopper--Pearson as published (RS+CP) and at its certified clip level (RS+CP clip). Median widths relative to the uninflated set are in parentheses and clip levels $B'<B$ are shown; the few coverages below $0.93$ (at most four per histogram) are counted in the first bin. The dashed line is $1-\alpha$.}
\label{fig:rarehist}
\end{figure}

\paragraph{Coverage across calibration sets.} Figure~\ref{fig:rarehist} shows the training-conditional coverage behind the widths of Table~\ref{tab:frontier} in the design with $\pi=0.02$ at $\delta=0.05$, where the certificates have comparable widths. The uninflated weighted quantile is centred at $0.95$ and fails for $52$--$60\%$ of the calibration sets. The Bentkus-inflated weighted quantile, at its certified clip level, shifts the whole distribution above $0.95$ with a standard deviation of $0.0036$, $0.0036$ and $0.0022$ at $m=2000$, $5000$ and $20000$. Rejection sampling as published keeps about $m/B\approx100$--$950$ points, and its coverage spreads from below $0.95$ to above $0.99$ at $m\le5000$ (standard deviation $0.0135$, $0.0116$ and $0.0061$), so it uses its $\delta$ budget on some calibration sets and is very conservative on others, while its median set is wider at every $m$. At its certified clip level $B'=1$ ($m\le5000$), rejection sampling keeps almost every point and is as concentrated as the weighted quantile; it is narrower at $m=2000$ ($1.200$ against $1.216$) and wider at $m=5000$ ($1.174$ against $1.151$). At $m=20000$ its certified clip level is $B$, so it coincides with rejection sampling as published.

\paragraph{Run-to-run variability of rejection sampling.} Rejection sampling and the weighted quantile differ in three respects that the width comparison does not show: the weighted quantile needs the functionals of $v$ while rejection sampling needs only $B$ (Section~\ref{sec:est} prices this); rejection sampling can be narrower at small $m$, as above; and rejection sampling is randomized, so that the set attached to a given calibration sample depends on the acceptance coins. We measure the last effect as follows. For each of the three shifts (tilt with $c=2$, rare subgroup with $\pi=0.01$, Gaussian) we draw two calibration sets of size $m=5000$ and, holding each fixed, rerun the rejection-sampling Clopper--Pearson procedure $300$ times with fresh acceptance coins ($\alpha=\delta=0.05$); the tail-adaptive weighted quantile with the Bernstein inflation is computed once, since it is a deterministic function of the same set. Table~\ref{tab:rerun} reports the range and interquartile range of the training-conditional coverage of the $300$ rejection-sampling sets and their widths relative to the weighted-quantile set.

\begin{table}[tb]
\centering\small
\caption{Rejection sampling rerun $300$ times on a fixed calibration set ($m=5000$, $\alpha=\delta=0.05$; two calibration sets per shift). Coverage is training-conditional coverage; widths are relative to the tail-adaptive Bernstein-inflated weighted quantile on the same set, whose coverage is in the third column.}
\label{tab:rerun}
\resizebox{\linewidth}{!}{\input{tables/tab_rerun.tex}}
\end{table}

On the mild tilt the reruns span coverages $0.944$--$0.969$, some below and some above the nominal $0.95$, with widths within $\pm8\%$ of the deterministic set; on the rare-subgroup and Gaussian shifts, where only $m/B\approx240$ and $150$ points are accepted on average, the coverages span $0.93$--$0.996$ and the widths $0.76$--$1.41$ times the deterministic set. The interquartile ranges ($0.003$--$0.017$) are of the same order as the standard deviation of the coverage across calibration sets, i.e.\ the acceptance coins contribute as much variability as the data.

\subsection{Certificates with estimated functionals}\label{sec:exp_est}

Table~\ref{tab:frontier} evaluates the certificates at the exact population functionals. We now apply Proposition~\ref{prop:est} as stated. For each replication, a fresh unlabeled source sample of size $n_s\in\{10^4,10^5,10^6\}$ is drawn and split in halves, the clip level and the grid points $s_L$ and $s_R$ (the latter by the margin rule \eqref{eq:margin}) are chosen on the first half, the empirical Bernstein bounds are computed on the second at level $\eta=\delta'/J$ with $\delta'=\delta/5$, and the inflation is computed at level $\delta-\delta'$ (details in Appendix~\ref{app:experiments}). Three shifts are used, the exponential tilt of Section~\ref{sec:exp_smooth} with $c=2$ ($B=4$, $\mathrm{CVaR}_\alpha\approx B$), the rare-subgroup design with $\pi=0.01$ ($B=21$, the running example) and the Gaussian shift ($B=33$), at $m\in\{5000,20000\}$. Table~\ref{tab:est} compares the estimated certificates with rejection sampling as published, which needs no functional beyond $B$ and pays nothing for estimation, and with the exact-functional certificates of Table~\ref{tab:frontier} on the same calibration draws.

\begin{table}[tb]
\centering\small
\caption{Certificates with functionals estimated from an independent unlabeled source sample of size $n_s$ (Proposition~\ref{prop:est}, halves for selection and certification, empirical Bernstein bounds, $\delta'=\delta/5$), $\alpha=\delta=0.05$, $150$ replications with a fresh source sample each; widths relative to the uninflated weighted quantile, clip levels in parentheses, widths are medians over all replications with $\cY$ counted as infinite, ``$\infty$'' when more than half are $\cY$, superscripts as in Table~\ref{tab:frontier}. ``RS'' is rejection sampling as published, which needs only $B$; the last two columns are the certificates of Table~\ref{tab:frontier} with exact functionals, on the same calibration draws. Every entry has empirical failure fraction $\le\delta$ up to Monte Carlo error (standard error $\approx0.018$); the narrowest finite-sample certificate in each row is in bold.}
\label{tab:est}
\resizebox{\linewidth}{!}{\input{tables/tab_est.tex}}
\end{table}

Validity is not at issue: the estimated weighted-quantile certificates fail in at most $0.7\%$ of the replications of any cell, as the union bound and the conservative inflation predict, and the Clopper--Pearson sets fail at their nominal rate. The test that certifies $\rho_{\alpha_+}$ for the Bentkus inflation, placed by the margin rule \eqref{eq:margin}, passes in every replication of the $B=4$ design at $n_s=10^6$ and in at least $97\%$ at $n_s\le10^5$; in the two large-$B$ designs it fails in every replication at $n_s=10^4$, where the range term $7B\log(2/\eta)/(3n_2)$ of the confidence bound alone exceeds $\alpha$, and in at most $1.3\%$ at $n_s\ge10^5$. Its outcome barely matters here: replacing $\rho_+$ by the fallback $\widetilde B_+$ changes the certified inflation by less than $10^{-4}$ in every cell (the term $(t-\gamma)\rho$ of Proposition~\ref{prop:bentkus} is small at $B\le33$; Example~\ref{ex:margin} shows a case with $B=100$ where it is not), so the Bentkus columns of Table~\ref{tab:est} are, to three decimals, those of the moment estimates alone. The question is the width. \emph{When $B$ is moderate}, estimation is essentially free: for $B=4$ the estimated Bentkus certificate is $1.080$ at $n_s=10^4$ against $1.077$ for the exact functionals, and the two agree to within three thousandths at $n_s\ge10^5$; here rejection sampling is the narrowest certificate ($1.05$--$1.06$), as the crossover condition of Section~\ref{sec:clip} predicts for $\mathrm{CVaR}_\alpha/B\approx1$. \emph{When $B$ is large}, the cost of estimation is the cost of certifying rare-event expectations. In the rare-subgroup design the certification half of a $10^4$-point source sample contains on average $2.5$ points of the subgroup, and the empirical Bernstein bounds on $\EE_Pv^2$ and on the tail functional have half-widths dominated by their range terms, of order $B^2\log(1/\eta)/n_2$: the estimated unclipped certificate is $1.29$ (Bentkus) against $1.10$ for the clipped one with exact functionals and $1.19$ for exact rejection sampling at $m=5000$. Clipping does not help at this sample size, because the bias bound $r_+$ on $\EE_P(v-B')^+$, which enters the over-coverage directly, cannot be certified below a few thousandths from $5000$ draws either; the median selected level is accordingly $B'=B$ at $n_s\le10^5$, and moves to $B'=8$ ($10$ with exact functionals) only at $n_s=10^6$. With $n_s=10^5$ the estimated Bentkus certificate is within a width point of rejection sampling at both sizes ($1.18$ vs.\ $1.19$ at $m=5000$; $1.074$ vs.\ $1.075$ at $m=20000$), and with $n_s=10^6$ it is narrower at both ($1.117$ vs.\ $1.182$; $1.055$ vs.\ $1.076$), within $1$--$2$ width points of the certificate with exact functionals ($1.096$, $1.047$). The Gaussian shift, with $B=33$ and $\TV(Q,P)=0.38$, behaves the same way with a larger range: $1.44$, $1.29$ and $1.21$ at $m=5000$ for $n_s=10^4,10^5,10^6$ against $1.18$ with exact functionals and $1.24$ for rejection sampling, and $1.144$, $1.106$, $1.083$ at $m=20000$ against $1.077$ and $1.09$--$1.10$. The clipped rejection-sampling set, which needs only the bias bound $r_+$, is within two width points of its exact-functional value at $n_s=10^6$ in both designs ($1.170$ vs.\ $1.154$ in the Gaussian shift; $1.114$ vs.\ $1.092$ in the rare-subgroup design, $B'=1$, where the certified bias of the extreme clip is again a rare-event expectation); part of the gap is the confidence budget $\delta'$ that the estimated version pays as well. The rule that emerges is the one stated after Proposition~\ref{prop:est}: the estimation error is comparable to $\sigma_\alpha^2$ when $n_2\approx B^2\log(1/\eta)/\sigma_\alpha^2$, which is about $10^4$ in both large-$B$ designs, and negligible at $10$--$100$ times that. Unlabeled source covariates are usually the cheapest data in a covariate-shift problem, so samples of this size are realistic, but the requirement is real, it grows with $B^2$, and it is the price at which the sharper, deterministic certificate of Sections~\ref{sec:upper}--\ref{sec:clip} is available without population knowledge; rejection sampling as published does not pay it.

\section{Discussion}\label{sec:discussion}

\paragraph{What the results enable.} The training-conditional coverage of the weighted quantile is governed by $\sigma_\alpha^2=\alpha^2K^2+(1-2\alpha)\Lambda(\alpha)$: by the $\chi^2$-divergence of the shift and by the likelihood ratio on the $\alpha$-fraction of the test population where it is largest, not by its supremum. Both ingredients are computable, so a practitioner who calibrates once can certify, deterministically and with every calibration point used at its weight, that the deployed set has miscoverage at most $\alpha$ for a prescribed fraction of calibration samples (Procedure~\ref{proc:wq}); the same certificate transfers to the weighted split conformal set of \citet{tibshirani2019conformal} through \eqref{eq:reduction}. The lower bound shows that the $K^2$ scale is intrinsic to the shift, and for the weighted quantile itself the full proxy $\sigma_\alpha^2$ is the exact asymptotic variance under the adversarial alignment of score tail and likelihood ratio.

\paragraph{When to use which procedure, and what it needs.} Section~\ref{sec:clip} reduces the comparison with rejection sampling, at leading order, to two numbers: the ratio $\sigma_\alpha^2/(\alpha(1-\alpha)B)$, which for small $\alpha$ is about $\mathrm{CVaR}_\alpha(v)/B$, and the constant $z_{1-\delta}^2/(2\log(1/\delta))$ that separates a distribution-free concentration bound from an exact binomial tail; clipping lets both procedures trade a computable bias for range, and the crossover sizes it predicts match the experiments. The weighted quantile is the narrower certificate when the supremum is a rare-event quantity, the calibration set is large or $\delta$ is small; clipped rejection sampling is narrower for small calibration sets and mild shifts, and on the real-data shifts of Section~\ref{sec:exp_real}, where $\mathrm{CVaR}_\alpha/B$ is about $0.5$ or more and $m$ is a few thousand. The inputs differ: rejection sampling needs only $B$ and is randomized, so that the certified set depends on the acceptance coins as much as on the data (Table~\ref{tab:rerun}); the weighted quantile is deterministic but needs $K^2$, $\Lambda(\alpha)$ and, for the Bentkus inflation, $\rho_\alpha$. When these are estimated from an unlabeled source sample, Proposition~\ref{prop:est} keeps the guarantee exact at a total failure probability $\delta$; the cost is that of certifying rare-event expectations, negligible for moderate $B$ and requiring source samples of order $10$--$100$ times $B^2\log(1/\eta)/\sigma_\alpha^2$ when $B$ is large (Section~\ref{sec:exp_est}), with the quantile $\rho_\alpha$ certified by the margin rule of Section~\ref{sec:est}. The importance-weighted learn-then-test procedure of \citet{almeida2025high} is deterministic as well and uses every point; its two concentration-based p-values are wider than the inflated weighted quantile in every cell of Table~\ref{tab:frontier}, and the Hoeffding--Bentkus variant is within a width point of WQ-Bk on the mild shifts of Table~\ref{tab:smooth}; this is a finite-sample gap that vanishes with $m$ and that its betting p-value may narrow.

\paragraph{Open questions.} The largest remaining slack is not in the concentration inequality but in the coupling: the empirical oracle inflation is $2$--$4$ times smaller than the certificate in the rare-event designs because the actual variance is close to $\alpha(1-\alpha)K^2$ while the bound must allow for $\Lambda(\alpha)$; a data-driven estimate of $\EE_Q[v\ind{S\ge q}]$ from the calibration sample would close most of this gap, but its own concentration involves $B$ again, and whether any procedure can avoid the coupling term is open. Bentkus's inequality removes $9$--$12\%$ of the inflation relative to Bernstein's, and its constant $e^2/2$ is not known to be sharp \citep{bentkus2004hoeffding}. The argument does not extend to full conformal or jackknife+ under covariate shift, where the threshold is not the quantile of a fixed sample; there the uniform-convergence bounds of \citet{pournaderi2026training} remain the tool. The estimated-ratio analysis identifies the score-projected error $\EE_P[v-\widetilde v\mid S]$ as the relevant quantity, which suggests estimating the ratio with the downstream score in mind, using unlabeled target data only. Finally, the lower bound is stated at constant probability; a change-of-measure version of the two-point argument should give its $\log(1/\delta)$-dependence for general threshold rules.

\begin{appendix}
\section{Proofs for Section~\ref{sec:upper}}\label{app:upper}

\begin{proof}[Proof of Theorem~\ref{thm:upper}]
Write $\tauhat=\tauhat(\beta)$. If $\beta-\gamma\le0$ the event is empty, so let $q:=q_Q(\beta-\gamma)<\infty$.

\emph{Step 1 (reduction to a fixed quantile).} If $P_e(\tauhat)>\alpha+\gamma$ then $F_Q(\tauhat)<\beta-\gamma$, hence $\tauhat<q$ (if $\tauhat\ge q$ then $F_Q(\tauhat)\ge F_Q(q)\ge\beta-\gamma$). In particular $\tauhat<\infty$, so $\sum_jv_j>0$ and $\Fhat(\tauhat)\ge\beta$. Since $\tauhat<q$ implies $\ind{S_i\le\tauhat}\le\ind{S_i<q}$,
\[
\sum_{i}\what_i\ind{S_i<q}\ge\Fhat(\tauhat)\ge\beta
\quad\Longleftrightarrow\quad
\overline U:=\frac1m\sum_{i\in\cI}U_i\ge0,
\]
where $U_i:=v_i(\xi_i-\beta)$ and $\xi_i:=\ind{S_i<q}$.

\emph{Step 2 (mean, range, second moment).} The $U_i$ are i.i.d.\ with $U_i\in[-\beta B,(1-\beta)B]$. Using $\EE_P[vh]=\EE_Q[h]$ and $\PP_Q(S<q_Q(\beta-\gamma))\le\beta-\gamma$,
\[
\EE U_i=\PP_Q(S<q)-\beta=:-t,\qquad t\in[\gamma,\beta].
\]
Moreover, with $\alpha=1-\beta$,
\begin{equation}\label{eq:secondmoment}
\begin{aligned}
\EE U_i^2=\EE_Q\big[v(\xi-\beta)^2\big]&=\alpha^2\EE_Q[v\xi]+(1-\alpha)^2\EE_Q[v(1-\xi)]\\
&=\alpha^2K^2+(1-2\alpha)\,\EE_Q\big[v\ind{S\ge q}\big],
\end{aligned}
\end{equation}
since $\EE_Q[v\xi]+\EE_Q[v(1-\xi)]=K^2$.

\emph{Step 3 (Hoeffding).} By Hoeffding's inequality,
\[
\PP(\overline U\ge0)\le\PP(\overline U-\EE\overline U\ge\gamma)\le e^{-2m\gamma^2/B^2}.
\]

\emph{Step 4 (Bernstein with the supremum).} Since $v\le B$ and $\PP_Q(\xi=1)=\beta-t$, for every $\alpha\in(0,1)$,
\begin{align*}
\Var(U_i)\le\EE U_i^2\le B\,\EE_Q\big[(\xi-\beta)^2\big]&=B\big[\alpha^2(\beta-t)+\beta^2(\alpha+t)\big]\\
&=B\big[\alpha(1-\alpha)+(1-2\alpha)t\big]\le B\big[\alpha(1-\alpha)+t\big],
\end{align*}
and, since $t\le1-\alpha$ and $B\ge1$,
\[
U_i-\EE U_i\le(1-\beta)B+t\le\alpha B+(1-\alpha)\le B .
\]
Bernstein's inequality gives
\[
\PP(\overline U-\EE\overline U\ge t)\le\exp\Big(-\frac{mt^2}{2B[\alpha(1-\alpha)+t]+\frac23Bt}\Big)=\exp\Big(-\frac{mt^2}{2B\alpha(1-\alpha)+\frac83Bt}\Big),
\]
and $t\mapsto t^2/(a+bt)$ is increasing, so $t\ge\gamma$ yields \eqref{eq:bern}.

\emph{Step 5 (Bernstein with the tail average).} Let $\alpha\le1/2$. By the Hardy--Littlewood inequality and the monotonicity of $F_v^{-1}$,
\[
\EE_Q[v\ind{S\ge q}]\le\Lambda\big(\PP_Q(S\ge q)\big)=\Lambda(\alpha+t),
\]
\[
\Lambda(\alpha+t)-\Lambda(\alpha)=\int_{1-\alpha-t}^{1-\alpha}F_v^{-1}(u)\,du\le t\rho_\alpha .
\]
Hence, by \eqref{eq:secondmoment},
\[
\Var(U_i)\le\sigma_\alpha^2+(1-2\alpha)\rho_\alpha t,\qquad U_i-\EE U_i\le\alpha B+t,
\]
and Bernstein's inequality gives
\[
\PP(\overline U-\EE\overline U\ge t)\le\exp\Big(-\frac{mt^2}{2\sigma_\alpha^2+2(1-2\alpha)\rho_\alpha t+\frac23(\alpha B+t)t}\Big)=\exp\Big(-\frac{mt^2}{D(t)}\Big),
\]
where $D(t):=2\sigma_\alpha^2+bt+\tfrac23t^2$ and $b:=2(1-2\alpha)\rho_\alpha+\frac23\alpha B$. Since
\[
\frac{d}{dt}\Big[\frac{t^2}{D(t)}\Big]\ \propto\ 2tD(t)-t^2D'(t)=4\sigma_\alpha^2t+bt^2\ \ge0,
\]
the bound is decreasing in $t$, and $t\ge\gamma$ yields \eqref{eq:tail}. Solving each bound for $\gamma$ at level $\delta$ gives \eqref{eq:gamma_m}.
\end{proof}

\begin{proof}[Proof of the bounds in Remark~\ref{rem:moments}]
Both bounds start from Step~1 of the proof of Theorem~\ref{thm:upper}, where $P_e(\tauhat(\beta))>\alpha+\gamma$ implies $\overline U\ge0$ with $U_i=v_i(\xi_i-\beta)$ and $\EE U_i=-t\le-\gamma$. For the second-moment bound, $|\xi_i-\beta|\le1$ gives $\EE U_i^2\le\EE_Pv^2\le K^2$, so Chebyshev's inequality yields
\[
\PP(\overline U\ge0)\le\PP(\overline U-\EE\overline U\ge\gamma)\le\frac{K^2}{m\gamma^2},
\]
which equals $\delta$ at $\gamma=K/\sqrt{m\delta}$. For the truncation bound, let
\[
U_i^B:=\min(v_i,B)(\xi_i-\beta),\qquad R_i:=(v_i-B)^+,
\]
so that $U_i-U_i^B=R_i(\xi_i-\beta)\le R_i$ and $\EE R_i=r_B$. Then
\[
\{\overline U\ge0\}\subseteq\{\overline U^B\ge-\kappa\}\cup\{\overline R>\kappa\},
\]
and Markov's inequality gives $\PP(\overline R>\kappa)\le r_B/\kappa$. The $U_i^B$ are i.i.d.\ with range $B$ and, since $\xi_i-\beta\ge-1$,
\[
\EE U_i^B=\EE U_i-\EE\big[R_i(\xi_i-\beta)\big]\le-\gamma+r_B,
\]
so Hoeffding's inequality gives
\[
\PP(\overline U^B\ge-\kappa)\le\PP\big(\overline U^B-\EE\overline U^B\ge\gamma-r_B-\kappa\big)\le\exp\Big(-\frac{2m(\gamma-r_B-\kappa)^2}{B^2}\Big).
\]
Finally
\[
r_B=\EE_P(v-B)^+=\EE_Q\big[(v-B)^+/v\big]=\EE_Q(1-B/v)^+\le\PP_Q(v>B).\qedhere
\]
\end{proof}

\begin{proof}[Proof of Corollary~\ref{cor:pac}]
(i) Here $\beta=1-\alpha+\gamma$, and $P_e(\tauhat(\beta))>\alpha$ means $F_Q(\tauhat(\beta))<1-\alpha=\beta-\gamma$. Step~1 of the previous proof with $q:=q_Q(1-\alpha)$ gives $\overline U\ge0$ with $U_i=v_i(\xi_i-\beta)$, $\xi_i=\ind{S_i<q}$, and
\[
\EE U_i=-t,\qquad t:=\beta-\PP_Q(S<q)\ge\gamma,\qquad \PP_Q(S\ge q)=1-\beta+t=\alpha+(t-\gamma).
\]
By \eqref{eq:secondmoment} (with $1-\beta=\alpha-\gamma$ in place of $\alpha$),
\begin{align*}
\EE U_i^2&=(\alpha-\gamma)^2K^2+(1-2\alpha+2\gamma)\,\EE_Q[v\ind{S\ge q}]\\
&\le\alpha^2K^2+(1-2\alpha+2\gamma)\big[\Lambda(\alpha)+(t-\gamma)\rho_\alpha\big],
\end{align*}
and $U_i-\EE U_i\le(1-\beta)B+t\le\alpha B+t$. Bernstein's inequality bounds $\PP(\overline U-\EE\overline U\ge t)$ by $\exp\{-mt^2/D(t)\}$ with
\[
D(t)=a'+b't+\tfrac23t^2,\qquad a':=2\alpha^2K^2+2(1-2\alpha+2\gamma)\big[\Lambda(\alpha)-\gamma\rho_\alpha\big]\ge0,
\]
\[
b':=2(1-2\alpha+2\gamma)\rho_\alpha+\tfrac23\alpha B\ge0
\]
(the sign of $a'$ follows from $\Lambda(\alpha)\ge\alpha\rho_\alpha\ge\gamma\rho_\alpha$); as in Step~5 the bound is decreasing in $t$, and evaluating at $t=\gamma$ gives
\[
D(\gamma)=2\alpha^2K^2+2(1-2\alpha+2\gamma)\Lambda(\alpha)+\tfrac23(\alpha B+\gamma)\gamma,
\]
the first bound in (i). The second follows from $K^2\le B$, $\Lambda(\alpha)\le\alpha B$, $\gamma\le\alpha\le1/2$ and $\gamma\le B$:
\[
D(\gamma)\le2\alpha^2B+2(1-2\alpha)\alpha B+4\alpha B\gamma+\tfrac23\alpha B\gamma+\tfrac23B\gamma\le2B\alpha(1-\alpha)+3B\gamma .
\]
Setting the first bound equal to $\delta$ and solving the quadratic
\[
(m-\tfrac23L)\gamma^2-Lb\gamma-2L\sigma_\alpha^2=0
\]
gives \eqref{eq:gamma_tail}.

(ii) Let $q':=q_Q(\beta+\gamma')$; since $\beta+\gamma'<1$ and $F_Q$ is continuous, $F_Q(q')=\beta+\gamma'$. If $P_e(\tauhat(\beta))<\alpha-\gamma-\gamma'$ then $F_Q(\tauhat(\beta))>F_Q(q')$, hence $\tauhat(\beta)>q'$, hence $\Fhat(q')<\beta$. Clearing the denominator (with equality when $\sum_jv_j=0$),
\[
\overline V:=\frac1m\sum_iV_i\le0,\qquad V_i:=v_i\big(\ind{S_i\le q'}-\beta\big),\qquad \EE V_i=F_Q(q')-\beta=\gamma' .
\]
As in \eqref{eq:secondmoment},
\[
\EE V_i^2=(1-\beta)^2K^2+(2\beta-1)\,\EE_Q[v\ind{S>q'}],\qquad \PP_Q(S>q')=\alpha-\gamma-\gamma'\le\alpha,
\]
and $1-\beta\le\alpha$, so
\[
\Var(V_i)\le\alpha^2K^2+(1-2\alpha+2\gamma)\Lambda(\alpha)=\sigma_\alpha^2+2\gamma\Lambda(\alpha),
\]
\[
-(V_i-\EE V_i)\le\beta B+\gamma'=(1-\alpha+\gamma)B+\gamma' .
\]
Bernstein's lower-tail inequality gives
\[
\PP(\overline V\le0)\le\exp\Big(-\frac{m\gamma'^2}{2\sigma_\alpha^2+4\gamma\Lambda(\alpha)+\frac23\big((1-\alpha+\gamma)B+\gamma'\big)\gamma'}\Big).\qedhere
\]
\end{proof}
\section{Proof of Theorem~\ref{thm:lower}}\label{app:lower}

Write $\abar:=1-\alpha$ and set
\begin{equation}\label{eq:eps}
\lambda:=\sqrt{\frac{\alpha}{2m\abar K^2}},\qquad \bar\varepsilon:=\lambda K^2=\sqrt{\frac{\alpha K^2}{2m\abar}},\qquad \gamma:=\frac{\bar\varepsilon\abar}{4}=\frac{1}{4\sqrt2}\sqrt{\frac{K^2\alpha\abar}{m}}=\gamma_m^{\rm lb}.
\end{equation}
The assumption $m\ge B^2/(4\alpha K^2)$ gives $\lambda^2B^2\le2\alpha^2/\abar$, hence, using $\alpha\le1/2$,
\begin{equation}\label{eq:conds}
\lambda B\le\min\Big\{1,\ \frac{2\alpha}{\abar}\Big\}\qquad\text{and}\qquad \gamma\le\frac{\bar\varepsilon\abar}{4}\le\frac{\lambda B\abar}{4}\le\frac{\alpha}{2},
\end{equation}
because $2\alpha^2/\abar\le1$ and $2\alpha^2/\abar\le4\alpha^2/\abar^2$ for $\alpha\le1/2$, and $\bar\varepsilon=\EE_Q[\lambda v]\le\lambda B$.

\emph{Step 1 (the two distributions).} We construct the conditional laws of the score directly; for concreteness take $\cY=[0,1]$ and the identity score $s(x,y)=y$, so that they are conditional laws of $Y$. Let
\[
h(t):=-\ind{0<t\le\abar}+\frac{\abar}{\alpha}\ind{\abar<t<1},\qquad\text{so that}\qquad \int_0^1h=0,
\]
and for $\varepsilon\in[0,1]$ let $G_\varepsilon$ be the distribution on $[0,1]$ with density $1+\varepsilon h$. Under $P^{(0)}_{S|X}$ let $S\mid X=x\sim G_0=\mathrm{Unif}(0,1)$ for all $x$; under $P^{(1)}_{S|X}$ let $S\mid X=x\sim G_{\varepsilon(x)}$ with $\varepsilon(x):=\lambda v(x)\in[0,1]$ by \eqref{eq:conds}. Both conditional laws have score densities bounded by $1+\abar/\alpha$. Since $G_\varepsilon$ is affine in $\varepsilon$, the marginal law of the score under $Q_1$ is
\[
\int G_{\varepsilon(x)}\,dQ_X(x)=G_{\bar\varepsilon},\qquad \bar\varepsilon=\EE_Q[\lambda v]=\lambda K^2,
\]
and under $Q_0$ it is $G_0$. Explicitly,
\[
G_{\bar\varepsilon}(t)=(1-\bar\varepsilon)t\quad(t\le\abar),\qquad G_{\bar\varepsilon}(\abar+u)=\abar(1-\bar\varepsilon)+\Big(1+\frac{\bar\varepsilon\abar}{\alpha}\Big)u\quad(u\in[0,\alpha]);
\]
thus $P_e^{(0)}(\tau)=1-G_0(\tau)$ and $P_e^{(1)}(\tau)=1-G_{\bar\varepsilon}(\tau)$.

\emph{Step 2 (a deterministic dichotomy).} For every $\tau\in\R\cup\{\pm\infty\}$,
\begin{equation}\label{eq:dichotomy}
1-G_0(\tau)\le\alpha-\gamma\qquad\text{or}\qquad 1-G_{\bar\varepsilon}(\tau)\ge\alpha+\gamma .
\end{equation}
If $\tau\le\abar$:
\[
1-G_{\bar\varepsilon}(\tau)\ge1-G_{\bar\varepsilon}(\abar)=\alpha+\bar\varepsilon\abar=\alpha+4\gamma .
\]
If $\tau\ge1$: $1-G_0(\tau)=0\le\alpha-\gamma$. If $\tau=\abar+u$ with $0<u<\alpha$: $1-G_0(\tau)=\alpha-u$, so the first alternative holds when $u\ge\gamma$, and when $u<\gamma=\bar\varepsilon\abar/4$,
\[
1-G_{\bar\varepsilon}(\tau)=\alpha+\bar\varepsilon\abar\Big(1-\frac u\alpha\Big)-u>\alpha+\frac{\bar\varepsilon\abar}{4}+\frac{\bar\varepsilon\abar}{2}\Big(1-\frac{\bar\varepsilon\abar}{2\alpha}\Big)\ge\alpha+\gamma,
\]
using $\bar\varepsilon\abar\le\lambda B\abar\le2\alpha$ from \eqref{eq:conds}.

\emph{Step 3 (indistinguishability).} $P_1$ and $P_0$ share the covariate law $P_X$ and differ only in the conditional score law, so
\begin{align*}
\chi^2(P_1\|P_0)&=\int\chi^2\big(G_{\varepsilon(x)}\,\|\,G_0\big)\,dP_X(x)=\int\varepsilon(x)^2\int_0^1h^2\,dP_X(x)\\
&=\frac{\abar}{\alpha}\,\lambda^2\,\EE_P[v^2]=\frac{\abar}{\alpha}\lambda^2K^2=\frac1{2m},
\end{align*}
since
\[
\int_0^1h^2=\abar+\alpha\Big(\frac{\abar}{\alpha}\Big)^2=\frac{\abar}{\alpha}.
\]
By $\KL\le\chi^2$, tensorization and Pinsker's inequality,
\[
\TV(P_0^{\otimes m},P_1^{\otimes m})\le\sqrt{\tfrac m2\KL(P_1\|P_0)}\le\tfrac12,
\]
and the same holds after tensoring with the law of $\xi$.

\emph{Step 4 (Le Cam).} Let $A:=\{1-G_0(\tauhat)\le\alpha-\gamma\}$; by \eqref{eq:dichotomy}, $A^c\subseteq\{1-G_{\bar\varepsilon}(\tauhat)\ge\alpha+\gamma\}$, so the left side of \eqref{eq:lecam} is at least
\[
\PP_{P_0}(A)+\PP_{P_1}(A^c)=1-\big(\PP_{P_1}(A)-\PP_{P_0}(A)\big)\ge1-\TV\big(P_0^{\otimes m},P_1^{\otimes m}\big)\ge\tfrac12 .
\]
Part (a) follows because one summand is at least $1/4$, and (b) because the PAC assumption gives $\PP_{P_1}(A^c)\le\delta$. \qed

For $Q_X$ concentrated on a set of $P_X$-mass $1/B$ (so $v\in\{0,B\}$ and $K^2=B$), the construction reduces to a two-point covariate space, $\varepsilon(x)=\lambda B$ on the support of $Q_X$, and the condition $m\ge B^2/(4\alpha K^2)=B/(4\alpha)$.
\begin{proof}[Proof of Proposition~\ref{prop:clt}]
Let $\beta=1-\alpha$ and $W_i:=F_Q(S_i)$. Under $Q$, $W$ is $\mathrm{Unif}(0,1)$ because $F_Q$ is continuous, and
\[
F_Q(\tauhat(\beta))=\widehat W(\beta):=\inf\Big\{w:\ \sum_iv_i\big(\ind{W_i\le w}-\beta\big)\ge0\Big\},
\]
the self-normalized weighted $\beta$-quantile of the $W_i$: indeed, since $F_Q$ is continuous and nondecreasing, $\{W_i\le w\}=\{S_i\le t_w\}$ with $t_w:=\sup\{t:F_Q(t)\le w\}$, and $t_w<\tauhat(\beta)$ if and only if $w<F_Q(\tauhat(\beta))$, so the two infima correspond. Hence
\[
P_e(\tauhat(\beta))-\alpha=-\big(\widehat W(\beta)-\beta\big).
\]
Fix $z\in\R$ and put
\[
w_m:=\beta+\frac{z}{\sqrt m},\qquad \Psi_m(w):=\frac1m\sum_iv_i\big(\ind{W_i\le w}-\beta\big),
\]
a nondecreasing right-continuous step function, so that $\{\widehat W(\beta)\le w_m\}=\{\Psi_m(w_m)\ge0\}$ on $\{\sum_iv_i>0\}$, whose probability tends to one. The summands $v_i(\ind{W_i\le w_m}-\beta)$ are i.i.d.\ and bounded by $B$, with mean $\EE_Q[\ind{W\le w_m}]-\beta=z/\sqrt m$ and variance
\begin{align*}
\EE_P\big[v^2(\ind{W\le w_m}-\beta)^2\big]-\frac{z^2}{m}\ &\to\ \alpha^2\EE_Q[v\ind{W\le\beta}]+\beta^2\EE_Q[v\ind{W>\beta}]\\
&=\alpha^2K^2+(1-2\alpha)\,\EE_Q[v\ind{S>q}]=:\varsigma^2,
\end{align*}
using $\{W>\beta\}=\{S>q\}$ $Q$-a.s. By the Lindeberg--Feller central limit theorem for this triangular array, $\sqrt m\,\Psi_m(w_m)\Rightarrow N(z,\varsigma^2)$, hence
\[
\PP\big(\sqrt m(\widehat W(\beta)-\beta)\le z\big)=\PP\big(\sqrt m\,\Psi_m(w_m)\ge0\big)+o(1)\ \to\ \Phi(z/\varsigma),
\]
which is the claim. For the example, $S=\ind A(X)+U$ has a continuous distribution under $Q$ with $Q(S>1)=Q(A)=\alpha$, so $q=1$ and $\{S>q\}=A$; if $v$ takes its largest values on $A$ then $\EE_Q[v\ind A]=\Lambda(\alpha)$ by the Hardy--Littlewood inequality, and $\varsigma^2=\sigma_\alpha^2$.
\end{proof}

\section{Proofs for Section~\ref{sec:estimated}}\label{app:estimated}

\begin{proof}[Proof of Theorem~\ref{thm:estimated}]
Conditionally on $\widehat v$, the calibration points are i.i.d.\ from $P$, $\widetilde v=d\widehat Q/dP\le\widetilde B$ with $\EE_P\widetilde v=1$, and $\breve\tau(\beta)$ is the weighted quantile \eqref{eq:tauhat} built with the self-normalized weights of $\widetilde v$. Theorem~\ref{thm:upper} and Corollary~\ref{cor:pac} therefore apply to $(\widehat Q,\widetilde B)$ and to $1-F_{\widehat Q}(\breve\tau)$, and $|P_e(\tau)-(1-F_{\widehat Q}(\tau))|\le\Delta$ for every $\tau$. For \eqref{eq:Delta_bounds}, let $g(S):=\EE_P[v-\widetilde v\mid S]$, so that
\[
F_Q(t)-F_{\widehat Q}(t)=\EE_P\big[g(S)\ind{S\le t}\big],\qquad \EE_Pg(S)=0;
\]
then
\[
-\EE_Pg^-\ \le\ \EE_P\big[g\ind{S\le t}\big]\ \le\ \EE_Pg^+,\qquad \EE_Pg^+=\EE_Pg^-=\tfrac12\EE_P|g| .
\]
The second inequality in \eqref{eq:Delta_bounds} is Jensen's, the equality is the definition of total variation, and the last inequality follows from
\[
\|\widetilde v-\widehat v\|_{L^1(P)}=|1-\EE_P\widehat v|=|\EE_P(v-\widehat v)| .\qedhere
\]
\end{proof}

\begin{proof}[Proof of Proposition~\ref{prop:interval}]
(i) With $A(t):=\sum_iv_i\ind{S_i\le t}$ and $C(t):=\sum_iv_i\ind{S_i>t}$, $\Fhat=A/(A+C)$. The map $(x,y)\mapsto x/(x+y)$ is nondecreasing in $x\ge0$ and nonincreasing in $y\ge0$, and $A_{\rm lo}\le A$, $C_{\rm hi}\ge C$, so $\breve F\le\Fhat$. As $t$ increases $A_{\rm lo}$ increases and $C_{\rm hi}$ decreases, so $\breve F$ is nondecreasing, and right-continuity is inherited from the indicators. Hence $\{\breve F\ge\beta\}\subseteq\{\Fhat\ge\beta\}$, $\breve\tau(\beta)\ge\tauhat(\beta)$, and the bounds follow from Theorem~\ref{thm:upper} with $v\le B_{\rm hi}$.

(ii) Write $A_{\rm lo}=A-D_A$ and $C_{\rm hi}=C+D_C$ with $D_A,D_C\ge0$, and let
\[
D:=\sum_i\big(v_{\rm hi}(Z_i)-v_{\rm lo}(Z_i)\big)\ge D_A+D_C,\qquad V:=A+C,\qquad \rho:=D/V\le\widehat\rho .
\]
Then
\[
\begin{aligned}
\breve F&=\frac{A-D_A}{A+C+D_C-D_A}\ge\frac{A-D_A}{V+D_C}\ge\frac{A-D}{V+D}\\
&=\frac{\Fhat-\rho}{1+\rho}=\Fhat-\frac{\rho(1+\Fhat)}{1+\rho}\ge\Fhat-2\widehat\rho,
\end{aligned}
\]
so $\{\Fhat\ge\beta+2\widehat\rho\}\subseteq\{\breve F\ge\beta\}$ and $\breve\tau(\beta)\le\tauhat(\beta+2\widehat\rho)$. On $\{\widehat\rho\le\rho_0\}$, $P_e(\breve\tau(\beta))\ge P_e(\tauhat(\beta+2\rho_0))$ and the proof of Corollary~\ref{cor:pac}(ii) at the deterministic level $\beta+2\rho_0$ gives the bound. Concentration of $\widehat\rho$ is Hoeffding's inequality for two bounded sample means.
\end{proof}

\begin{proof}[Proof of Lemma~\ref{lem:clipbias}]
Since $v_{B'}\le v$ and $\mu_{B'}\le1$,
\[
\Big|v-\frac{v_{B'}}{\mu_{B'}}\Big|\le(v-v_{B'})+v_{B'}\Big(\frac1{\mu_{B'}}-1\Big),
\]
and taking expectations,
\[
2\Delta_{B'}\le(1-\mu_{B'})+\mu_{B'}\Big(\frac1{\mu_{B'}}-1\Big)=2(1-\mu_{B'}),\qquad 1-\mu_{B'}=\EE_P(v-v_{B'})=\EE_P(v-B')^+ .
\]
The variational formula is the Rockafellar--Uryasev representation of the conditional value at risk \citep{rockafellar2000optimization}: for a random variable $W$ with quantile function $F^{-1}$,
\[
\int_{1-a}^1F^{-1}(u)\,du=\inf_{t}\big\{at+\EE(W-t)^+\big\},
\]
attained at $t=F^{-1}(1-a)$; apply it to $W=v$ under $Q$ and use $\EE_Q(v-s)^+=\EE_P[v(v-s)^+]$.
\end{proof}

\begin{proof}[Proof of Proposition~\ref{prop:est}]
Let $E$ be the event that the $J$ upper confidence bounds computed from $I_2$ hold; by the union bound, $\PP(E)\ge1-\delta'$. Since $B'$, $s_L$ and $s_R$ are functions of $I_1$, which is independent of $I_2$, each bound is applied to a fixed function conditionally on $I_1$, so this is legitimate. On $E$, the following hold. By Lemma~\ref{lem:clipbias},
\[
\Delta_{B'}\le1-\mu_{B'}=\EE_P(v-B')^+\le r_+,
\]
hence
\[
\mu_{B'}\ge\mu_-,\qquad \alpha':=\alpha-\Delta_{B'}\ge\alpha_+,\qquad \|\widetilde v_{B'}\|_\infty\le\frac{B'}{\mu_{B'}}\le\widetilde B_+ .
\]
The second moment of $\widetilde v_{B'}$ is $\EE_Pv_{B'}^2/\mu_{B'}^2\le K^2_+$, using also the trivial bound $K^2\le\|\widetilde v_{B'}\|_\infty$. By Lemma~\ref{lem:clipbias} applied to $Q_{B'}$ at $t=s_L/\mu_{B'}$, and using $\Lambda(a)\le\min(aB,K^2)$,
\[
\Lambda_{B'}(\alpha_+)\le\frac{\alpha_+s_L}{\mu_{B'}}+\frac{\EE_P\big[v_{B'}(v_{B'}-s_L)^+\big]}{\mu_{B'}^2}\le\Lambda_+ .
\]
If $U_\eta(v_{B'}\ind{v_{B'}>s_R})\le\alpha_+\mu_-$ then
\[
Q_{B'}\Big(\widetilde v_{B'}>\frac{s_R}{\mu_{B'}}\Big)=\frac{\EE_P\big[v_{B'}\ind{v_{B'}>s_R}\big]}{\mu_{B'}}\le\alpha_+,
\]
so the $(1-\alpha_+)$-quantile of $\widetilde v_{B'}$ under $Q_{B'}$ is at most $s_R/\mu_{B'}\le\rho_+$ (and $\rho_+\le\widetilde B_+$ in both cases since $s_R\le B'$). Thus on $E$ every functional of the pair $(Q_{B'},\widetilde v_{B'})$ at the target $\alpha_+$ is bounded above by its estimate. Now condition on the source sample and apply Corollary~\ref{cor:pac}(i), or Proposition~\ref{prop:bentkus}, to the pair $(P,Q_{B'})$ at the target $\alpha_+\le1/2$ and level $\delta-\delta'$. Their proofs use $K^2$, $\Lambda(\alpha_+)$ and the range only as upper bounds on the variance and the range of the summands $U_i$ in Bernstein's or Bentkus's inequality, and $\rho_{\alpha_+}$ only as an upper bound on the slope of the concave function $\Lambda$ at $\alpha_+$; all three inequalities remain valid when these quantities are replaced by larger ones (in the proof of Corollary~\ref{cor:pac}(i), $\rho_{\alpha_+}$ is used only in the intermediate monotonicity step, with its true value, and the final bound does not involve it, which is why $J=3$ suffices there). Hence, on $E$, the weighted quantile $\breve\tau$ with weights $v_{B'}(Z_i)$ at level $1-\alpha_++\widehat\gamma=1-\alpha+r_++\widehat\gamma$ satisfies
\[
\PP\big(1-F_{Q_{B'}}(\breve\tau)>\alpha_+\ \big|\ \text{source}\big)\le\delta-\delta',
\]
and whenever the event inside does not occur,
\[
P_e(\breve\tau)=1-F_Q(\breve\tau)\le1-F_{Q_{B'}}(\breve\tau)+\Delta_{B'}\le\alpha_++r_+=\alpha .
\]
Integrating over the source sample,
\[
\PP(P_e>\alpha)\le\PP(E^c)+(\delta-\delta')\le\delta .
\]
For rejection sampling, on $E$ the accepted points are i.i.d.\ from $Q_{B'}$ and the Clopper--Pearson set at level $\alpha_+$ and confidence $1-(\delta-\delta')$ has $1-F_{Q_{B'}}\le\alpha_+$ with that probability, and the same displacement by $\Delta_{B'}\le r_+$ concludes.
\end{proof}

\subsection{Consistency of the margin rule}\label{app:margin}

\begin{proposition}[Margin rule]\label{prop:margin}
Fix $B'\in[1,B]$ with $\alpha'':=\alpha-\EE_P(v-B')^+>0$, and let
\[
T(s):=\EE_P\big[v_{B'}\ind{v_{B'}>s}\big],\qquad \theta:=\alpha''\mu_{B'},\qquad s^\star:=\inf\{s\ge0:\ T(s)\le\theta\},
\]
so that $s^\star/\mu_{B'}$ is the $(1-\alpha'')$-quantile $\rho_{\alpha''}$ of $\widetilde v_{B'}$ under $Q_{B'}$. Suppose $s^\star<B'$, $P(v_{B'}=s^\star)=0$, and the law of $v_{B'}$ under $P$ gives positive mass to every interval in a neighbourhood of $s^\star$. Let the upper confidence bounds $U_\eta$, on both halves, be the empirical Bernstein bounds \eqref{eq:eb}, at a fixed level $\eta$ (a valid but uninformative bound, such as $U_\eta(g)=R$, satisfies Proposition~\ref{prop:est} but never certifies $\rho$). If $n_1,n_2\to\infty$ with $n_1/n_2$ bounded and $s_R$ is chosen by \eqref{eq:margin}, then the probability that the test defining $\rho_+$ in Proposition~\ref{prop:est} passes tends to one, and $\rho_+\to\rho_{\alpha''}$ in probability. Since $\alpha_+\to\alpha''$, this is the value the certificate targets.
\end{proposition}

The regularity conditions say that $T$ is continuous at $s^\star$ and strictly decreasing near it, and they hold for the smooth shifts of Section~\ref{sec:experiments}. An atom of $v_{B'}$ is not by itself an obstacle; what the conditions exclude is a tail expectation exactly at the test boundary, a jump of $T$ that lands on $\theta$, where the probability that the constraint in \eqref{eq:margin} holds at $s^\star$ itself tends to zero and the rule selects the smallest observed value above $s^\star$ that satisfies \eqref{eq:margin}, still with a test that passes with probability tending to one; the same argument then gives
\[
\rho_+\ \to\ \frac{1}{\mu_{B'}}\inf\big\{s>s^\star:\ P\big(v_{B'}\in(s^\star,s]\big)>0\big\},
\]
which exceeds $\rho_{\alpha''}$ only when the support of $v_{B'}$ has a gap to the right of $s^\star$. For a two-atom ratio with values $v_{\rm lo}<v_{\rm hi}$ and $\EE_P[v_{B'}\ind{v_{B'}=v_{\rm hi}}]\ne\theta$, the same argument (with the constraint evaluated at the two support points, across which $T$ jumps from $\mu_{B'}$ to $\EE_P[v_{B'}\ind{v_{B'}=v_{\rm hi}}]$ to $0$) shows that the rule selects the right atom and the test passes with probability tending to one.

\begin{example}\label{ex:margin}
Let $v$ under $P$ be the mixture
\[
0.9849\,\mathrm{Unif}[c-0.1,c+0.1]+0.015\,\mathrm{Unif}[1.9,2.1]+0.0001\,\mathrm{Unif}[99.9,100.1]
\]
with $c=0.96/0.9849$
(unit mean, $B=100.1$), $\alpha=0.05$, $\delta-\delta'=0.04$, $\eta=0.0025$ and $B'=B$. Then $K^2=2.00$, $\Lambda(\alpha)=1.07$ and $\rho_\alpha=1.07$; $T$ equals $0.04$ exactly on $[c+0.1,1.9)\approx[1.075,1.9)$. At $m=5000$ the Bentkus inflation is $0.0334$ with these functionals and $0.0367$, ten percent more, with $\rho_\alpha$ replaced by $B$; every $\rho\le20$ gives an inflation within $0.0001$ of the first value. With $n_1=n_2=5\times10^5$ the rule \eqref{eq:margin} selects $s_R\approx1.94$, past the flat part, and the test on $I_2$ passed in $199$ of $200$ replications, against $89$ of $200$ for the rule without margin; with $n_1=n_2=5\times10^4$ the range term $7B\log(2/\eta)/(3n_2)\approx0.031$ of the confidence bound leaves no room below $\alpha$, and the certified $\rho_+$ is in the upper uniform component under either rule. This is the cost of certifying a rare-event expectation with $B=100$, not of the rule.
\end{example}

\begin{proof}[Proof of Proposition~\ref{prop:margin}]
Write $n=n_1$, and let $F$ be the distribution function of $v_{B'}$ under $P$, $\widehat F_1$ its empirical version on $I_1$, $\bar g_1(s)$ and $\widehat{\Var}_1(s)$ the sample mean and variance of $v_{B'}\ind{v_{B'}>s}$ on $I_1$, and $V(s):=\Var_P(v_{B'}\ind{v_{B'}>s})$. Integrating by parts,
\begin{align*}
T(s)&=s\big(1-F(s)\big)+\int_s^{B'}\big(1-F(u)\big)\,du,\\
T_2(s)&:=\EE_P\big[v_{B'}^2\ind{v_{B'}>s}\big]=s^2\big(1-F(s)\big)+\int_s^{B'}2u\big(1-F(u)\big)\,du,
\end{align*}
and the same identities hold for the empirical versions with $\widehat F_1$; hence, with $\widehat T_2$ the empirical second moment,
\[
\sup_s|\bar g_1(s)-T(s)|\le2B'\|\widehat F_1-F\|_\infty,\qquad \sup_s|\widehat T_2(s)-T_2(s)|\le2B'^2\|\widehat F_1-F\|_\infty,
\]
and, since $\widehat{\Var}_1=\frac{n}{n-1}(\widehat T_2-\bar g_1^2)$ with $\bar g_1,T\le B'$,
\[
\sup_s|\widehat{\Var}_1(s)-V(s)|\le\frac{6n}{n-1}B'^2\|\widehat F_1-F\|_\infty+\frac{B'^2}{n-1};
\]
all are $O_p(n^{-1/2})$ by the Dvoretzky--Kiefer--Wolfowitz inequality \citep{dvoretzky1956asymptotic}. The half-width of $U^{(1)}_\eta$ and the margin $\kappa_1$ are at most
\[
B'\sqrt{\frac{2\log(2/\eta)}{n}}+\frac{7B'\log(2/\eta)}{3(n-1)}\qquad\text{and}\qquad B'\sqrt{\frac{2\log n}{n}}
\]
uniformly in $s$, and $\theta_1:=\alpha^{(1)}_+\mu^{(1)}_-\to\theta$ in probability, since $r^{(1)}_+$ is a consistent estimate of $\EE_P(v-B')^+$ plus a half-width that vanishes (and $r_+^{(1)}=0$ when $B'=B$). Consequently
\begin{equation}\label{eq:margin_unif}
\sup_s\big|U^{(1)}_\eta\big(v_{B'}\ind{v_{B'}>s}\big)+\kappa_1(s)-T(s)\big|+|\theta_1-\theta|=:e_n\to0\quad\text{in probability.}
\end{equation}

\emph{Step 1: $s_R\to s^\star$.} $T$ is right-continuous and nonincreasing, so $T(s^\star)\le\theta\le T(s^\star-)$, and $P(v_{B'}=s^\star)=0$ gives $T(s^\star)=\theta$; the positive-mass condition makes $T$ strictly decreasing on some $[s^\star-\varepsilon_0,s^\star+\varepsilon_0]\subset(0,B')$. Fix $\varepsilon\in(0,\varepsilon_0)$ and put
\[
c_\varepsilon:=\min\big\{\theta-T(s^\star+\varepsilon/2),\ T(s^\star-\varepsilon)-\theta\big\}>0 .
\]
With probability tending to one, $I_1$ contains a value $s'\in(s^\star+\varepsilon/2,s^\star+\varepsilon)$ and $e_n<c_\varepsilon/2$. On this event the constraint in \eqref{eq:margin} holds at $s'$, because its left side is at most
\[
T(s')+e_n\le\theta-c_\varepsilon+e_n<\theta-e_n\le\theta_1,
\]
so $s_R\le s'<s^\star+\varepsilon$; and it fails at every $s\le s^\star-\varepsilon$, because there the left side is at least
\[
T(s)-e_n\ge\theta+c_\varepsilon-e_n>\theta+e_n\ge\theta_1,
\]
so $s_R\ge s^\star-\varepsilon$. Hence $s_R\to s^\star$ in probability, and by the uniform convergence of $\widehat{\Var}_1$ and the continuity of $V$ at $s^\star$ (no atom there), $\widehat{\Var}_1(s_R)\to V(s^\star)$ in probability. Moreover $V(s^\star)>0$: otherwise $v_{B'}\ind{v_{B'}>s^\star}$ is a.s.\ constant, either $0$, contradicting $T(s^\star)=\theta>0$, or a constant $c>s^\star$ attained with probability one, giving $T(s^\star)=\mu_{B'}>\alpha''\mu_{B'}=\theta$.

\emph{Step 2: the test passes.} Let $\bar g_2$, $\widehat{\Var}_2$ denote the $I_2$-versions and $\theta_2:=\alpha_+\mu_-$ the threshold of the test in Proposition~\ref{prop:est}. Conditionally on $I_1$ the point $s_R$ is fixed, so $\bar g_2(s_R)-T(s_R)=O_p(n_2^{-1/2})$ by Chebyshev's inequality, and $\bar g_1(s_R)-T(s_R)=O_p(n^{-1/2})$ by the uniform bound above. The half-widths of $U^{(2)}_\eta$ and $U^{(1)}_\eta$ at $s_R$ are each $O(n^{-1/2})$ deterministically (their variance terms are at most $B'\sqrt{2\log(2/\eta)/n_j}$), and $\theta_2-\theta_1=O_p(n^{-1/2})$ for the same reasons applied to $(v-B')^+$. Since $s_R$ satisfies \eqref{eq:margin} when the set there is nonempty, which happens with probability tending to one by Step~1,
\begin{align*}
U^{(2)}_\eta\big(v_{B'}\ind{v_{B'}>s_R}\big)-\theta_2&=\big[U^{(1)}_\eta\big(v_{B'}\ind{v_{B'}>s_R}\big)-\theta_1\big]+O_p(n^{-1/2})\\
&\le-\kappa_1(s_R)+O_p(n^{-1/2}),
\end{align*}
and by Step~1
\[
\sqrt n\,\kappa_1(s_R)=\sqrt{2\widehat{\Var}_1(s_R)\log n}\ \to\ \infty\quad\text{in probability.}
\]
Hence the right side is negative with probability tending to one, i.e.\ the test passes, and on that event $\rho_+=s_R/\mu_-\to s^\star/\mu_{B'}=\rho_{\alpha''}$ because $\mu_-=1-r_+\to\mu_{B'}$.
\end{proof}

\section{Experimental details}\label{app:experiments}

\paragraph{Tail functionals.} In the synthetic designs $v$ depends on $x_1$ alone, so the law of $v$ under $P$ is available exactly: it has two atoms in the rare-subgroup designs ($v=1-\pi$ with probability $1-p_A$ and $v=1-\pi+\pi/p_A$ with probability $p_A$), and for the tilts it is represented by midpoint quadrature on $10^6$ points of the range of $x_1$ (uniform, or truncated normal with the exact normalizer $\EE_Pe^{cX_1}$); Sections~\ref{sec:exp_smooth} and \ref{sec:exp_frontier}--\ref{sec:exp_est} use this representation (a Monte Carlo estimate from $2\times10^6$ draws would have a relative standard error of about $4.6\%$ for $\Lambda(0.05)$ in the $B=51$ design, which is not negligible for the width comparisons there). Section~\ref{sec:exp_estimated} uses only the sup-$B$ inflation with the estimated ratio, and Section~\ref{sec:exp_real} uses the pool. In all cases the functionals are computed from values $v_i$ with probabilities $p_i$ (equal to $1/N$ for a sample) through $\EE_Q[g(v)]=\EE_P[v\,g(v)]$: $K^2=\sum_ip_iv_i^2$, $\rho_\alpha$ is the $(1-\alpha)$-quantile of $\{v_i\}$ under the probabilities $p_iv_i$, and $\Lambda(\alpha)=\sum_ip_iv_i^2\ind{v_i>\rho_\alpha}+\big(\alpha-\sum_ip_iv_i\ind{v_i>\rho_\alpha}\big)\rho_\alpha$.

\paragraph{The worst case (Remark~\ref{rem:beta}).}\label{app:worst} As a numerical check of Remark~\ref{rem:beta}, we use the construction of Theorem~\ref{thm:lower}, $v\in\{0,B\}$, with continuous scores, where all coverages have the exact Beta law; this allows $4000$ replications per configuration and an ``oracle'' inflation obtained by bisection. We record, for $\alpha=\delta=0.05$, the failure fraction and the median excess coverage of each procedure as a function of $m\in\{500,\dots,16000\}$ for $B\in\{1,4,16\}$ (the full results are stored with the code). Three things stand out. First, the uninflated weighted quantile fails for $50$--$57\%$ of the calibration sets in every configuration, as it must, since its coverage is centred at $1-\alpha$; and the empirical standard deviation of its coverage equals $\sqrt{B\alpha(1-\alpha)/m}$ to within $4\%$ in the $16$ configurations with $m/B\ge125$, and to within $15\%$ in the two with fewer than about $65$ weighted points on average, confirming Remark~\ref{rem:beta}. Second, the sup-$B$ Bernstein inflation \eqref{eq:gamma_m} (the tail-adaptive version \eqref{eq:gamma_tail} has the same leading term here, since $\Lambda(\alpha)=\alpha B$, and a slightly smaller lower-order term: $0.0121$ against $0.0141$ for $B=1$, $m=2000$) meets the target ($\le0.7\%$ failures) with an inflation $\gamma_m(\delta)$ between $1.6$ and $2.2$ times the oracle inflation where it is feasible, the ratio decreasing in $m$; the oracle itself is $1.5$--$1.7$ times $\sqrt{B\alpha(1-\alpha)/m}$ for $m/B\ge125$, i.e.\ about $z_{0.95}$ standard deviations, as the Beta law predicts. Third, the rejection-sampling set of \citet{park2022pac} essentially coincides with the oracle in this example (median coverages within $0.002$ for $m/B\ge250$) (the accepted points are exactly the $N$ informative ones and the Clopper--Pearson bound is exact for a binomial count): in the worst case the two procedures have the same effective sample size, and the remaining slack of the Bernstein bound is in the concentration inequality, not in the procedure. The Hoeffding inflation, whose leading term scales with $B$ rather than $\sqrt{B\alpha(1-\alpha)}$, is infeasible ($\gamma\ge\alpha$, set equal to $\cY$) for $B=4$ unless $m>8000$ and for $B=16$ throughout; the Bernstein inflation is infeasible for $B=16$ and $m\le4000$, where even the oracle cannot reach $\delta$ at $m=500$. Implementation: For each replication, $N\sim\mathrm{Bin}(m,1/B)$ without truncation, and the coverage of the weighted quantile at level $\ell$ is drawn from $\mathrm{Beta}(k,N+1-k)$ with $k=\lceil\ell N\rceil$ when $N\ge1$ and $k\le N$, and set to $1$ otherwise (the set is all of $\cY$; $N=0$ has probability at most $(1-1/16)^{500}<10^{-13}$ in the configurations used and did not occur). RS+CP accepts exactly the $N$ points and allows $k^\star$ exceedances, the largest $k$ with $\mathrm{Beta}^{-1}(1-\delta;k+1,N-k)\le\alpha$, giving coverage $\mathrm{Beta}(N-k^\star,k^\star+1)$. Each replication draws $N$ and one uniform $U$ once; the coverage at level $\ell$ is the $U$-quantile of the corresponding Beta law, so that all levels, including those visited by the bisection for the oracle, are evaluated on the same draws.

\paragraph{Smooth shift and estimated weights (Sections~\ref{sec:exp_smooth}--\ref{sec:exp_estimated}).} Samples from $Q$ are obtained by rejection sampling from $P$ with acceptance probability $v(x)/B$. In Section~\ref{sec:exp_estimated}, $\Delta$, $\TV(Q,\widehat Q)$ and the score-projected error are computed on $2\times10^5$ points from $P$, the last one by conditioning on $200$ equal-probability bins of the score.

\paragraph{Run-to-run variability (Section~\ref{sec:exp_frontier}, Table~\ref{tab:rerun}).} Calibration sets are drawn from $P$ as in Section~\ref{sec:exp_frontier}; the $300$ reruns differ only in the acceptance coins of rejection sampling, and coverage is evaluated on the same $2\times10^6$ reference scores.

\paragraph{Importance-weighted learn-then-test.} The fixed grid of thresholds consists of the $999$ quantiles $i/1000$ of the score under $Q$, computed from the $2\times10^6$ reference scores and therefore independent of the calibration data; thresholds are tested in decreasing order and the procedure returns the smallest threshold such that every larger one was rejected at level $\delta$ (all of $\cY$ if the first is not). The Hoeffding--Bentkus p-value is
\[
\min\Big\{\exp\big(-m\,h(\min\{\widehat R/B,\alpha/B\},\alpha/B)\big),\ e\,F_{\mathrm{Bin}(m,\alpha/B)}\big(\lceil m\widehat R/B\rceil\big)\Big\}
\]
with $h$ the Bernoulli relative entropy and $\widehat R$ the mean weighted loss (the Hoeffding branch is a lower-tail bound and equals one when $\widehat R\ge\alpha$), as in Appendix~B of \citet{almeida2025high}.

\paragraph{Bentkus inflation.} The supremum in Proposition~\ref{prop:bentkus} is handled as follows. For a candidate $\gamma$, $[\gamma,\beta]$ is covered by $1999$ geometric intervals $[t_j,t_{j+1}]$ ($2000$ points). The true deficit $t$ of the proof lies in one of them, and on $[t_j,t_{j+1}]$ the centred summands have variance at most $\sigma^2(t_{j+1})$ and range at most $b(t_{j+1})$, since $\sigma^2(t)=(\alpha-\gamma)^2K^2+(1-2\alpha+2\gamma)[\Lambda(\alpha)+(t-\gamma)\rho_\alpha]$ and $b(t)=(\alpha-\gamma)B+t$ are nondecreasing; Bentkus's theorem applied directly to these summands with those upper bounds, and with the smaller threshold $mt_j\le mt$, gives $\PP(\overline U-\EE\overline U\ge t)\le\mathrm{Bk}_m(mt_j;\sigma^2(t_{j+1}),b(t_{j+1}))$. The maximum of the right-hand side over $j$ therefore bounds the probability in Proposition~\ref{prop:bentkus} whatever the value of $t$ (no monotonicity of the function $\mathrm{Bk}_m$ in its variance and range arguments is used or claimed). The implemented inflation is found by bisection so that the maximum interval bound is at most $\delta$, and is therefore a valid inflation; it is a different certificate from the supremum-based $\gamma_m^{\rm Bk}(\delta)$ of Proposition~\ref{prop:bentkus} (neither dominates the other in general, since $\mathrm{Bk}_m$ is not monotone in its variance and range arguments), and it exceeds the pointwise value at $t=\gamma$ by less than $0.1\%$ in our configurations. Since the bound need not be monotone in $\gamma$, the bisection is not guaranteed to return the smallest feasible $\gamma$; for every unclipped and clipped Bentkus inflation of Table~\ref{tab:frontier} we checked that the interval bound exceeds $\delta$ at all $399$ points of a uniform grid of $(0,\gamma)$, so the reported values are minimal up to that grid for the implemented interval certificate (this says nothing about the supremum-based certificate of Proposition~\ref{prop:bentkus}). The range $B$ entering all population certificates is the known supremum of $v$ (for the clipped ones $\min(B,B')/\mu_{B'}$), not the maximum over quadrature nodes or sample points, which is not a certified supremum.

\paragraph{Widths.} All reported widths are medians over all replications, with a set equal to $\cY$ counted as infinite width, relative to the median width of the uninflated weighted quantile ($0.9$-quantiles, stored with the results, use the inverted-CDF convention so that they are infinite when more than $10\%$ of the sets are $\cY$). For Table~\ref{tab:frontier}, bootstrap standard errors are the standard deviation of $300$ bootstrap medians computed on the same basis; when any bootstrap median is infinite the standard error is reported as $\infty$ together with the frequency of infinite bootstrap medians (this occurs only in cells whose median is itself $\infty$). No bootstrap summaries are computed for Table~\ref{tab:est}. A median is reported as $\infty$ when more than half of the replications give $\cY$.

\paragraph{Estimated functionals (Section~\ref{sec:exp_est}).} For each replication a fresh source sample of size $n_s$ is drawn and split in halves. On the first half, for each clip level $B'$ on the geometric grid of Section~\ref{sec:exp_frontier}, the certificate of Proposition~\ref{prop:est} is computed with the first half in the role of $I_2$ (Bernstein inflation for the weighted quantile; the Bentkus version reuses the Bernstein selection), and $B'$ minimizing $r_++\widehat\gamma$ is selected, together with the grid points: $s_L$, used in $\Lambda_+$, is the $v_{B'}$-weighted empirical $(1-\alpha_+)$-quantile of $v_{B'}$ on the first half; $s_R$, used in the test for $\rho_+$ (Bentkus only), is chosen by the margin rule \eqref{eq:margin} on the first half, i.e.\ the smallest distinct value of $v_{B'}$ at which the first-half upper confidence bound on $\EE_P[v_{B'}\ind{v_{B'}>s}]$ plus the margin $\kappa_1(s)$ is at most $\alpha^{(1)}_+\mu^{(1)}_-$ (the clip level $B'$ if none is), with $\alpha^{(1)}_+,\mu^{(1)}_-$ the first-half versions of $\alpha_+,\mu_-$. The outcome of the test on the second half (pass or fail, among the replications in which the $\rho$-certification test was evaluated, i.e.\ those with $\alpha_+>0$ on both halves, whether or not the resulting inflation is finite) is recorded with the results. On the second half the upper confidence bounds are the empirical Bernstein bounds of \citet{maurer2009empirical} at level $\eta=\delta'/J$ with $\delta'=\delta/5$, and the inflation is computed at level $\delta-\delta'$. Training-conditional coverage in this experiment is evaluated on $2\times10^6$ test scores from $Q$, since the certified excesses at $m=20000$ are of the order of the Monte Carlo error of a $3\times10^5$-point reference.

\paragraph{Real data (Section~\ref{sec:exp_real}).} Tilt vectors are those of \citet[App.~E]{pournaderi2026training}: $\beta=(0.5,1,0,\dots)$ for Wine quality, $(0,10,10,0,\dots)$ on (length, diameter) for Abalone, $(0.01,0.01,0,\dots)$ for Concrete and $(0.2,0,0,0)$ for CCPP, applied to the raw features; $v(x)=e^{x^\top\beta}/\mathrm{mean}_{\rm pool}(e^{X^\top\beta})$, normalized to unit mean on the pool (the data left after the training split), and $B=\max_{\rm pool}v$. The random forest has $200$ trees with at least $5$ samples per leaf. The scaled tilts multiply $\beta_{\rm data}$ by the factor that makes $\max_{\rm pool}v=4$. The sample sizes are $(n,n_{\rm train},m)=(6497,1500,3000)$ for Wine quality, $(4177,1000,2000)$ for Abalone, $(1030,300,500)$ for Concrete and $(9568,2000,4000)$ for CCPP; the pool is the remainder of the data after the training split.
\end{appendix}

\begin{acks}[Acknowledgments]
Claude Fable 5.1 (Anthropic) and GPT 6 Astra were used in the preparation of this manuscript; the author thanks Javad Alipanah for providing access to them. The author retains full responsibility for the mathematical content and conclusions of the paper.
\end{acks}

\begin{supplement}
\stitle{Code and results}
\sdescription{Python code that reproduces every table and figure, together with the stored simulation results, is available at \url{https://github.com/mehrdad-pournaderi/Sharp-training-conditional-coverage-for-conformal-prediction-under-covariate-shift}.}
\end{supplement}

\bibliographystyle{imsart-nameyear}
\bibliography{refs}

\end{document}

%% file: tables/tab_smooth.tex
\begin{tabular}{cccccccccc}\toprule
\multicolumn{10}{l}{(a) fraction of calibration sets with $P_e>\alpha$}\\
$B$ & $\mathrm{CVaR}_\alpha$ & none & WQ & WQ-Bk & \shortstack{WQ-Bk\\clip} & RS+CP & \shortstack{RS+CP\\clip} & IW-HB & IW-EB\\\midrule
1.00 & 1.00 & 0.528 & 0.002 & 0.012 & 0.012 & 0.047 & 0.047 & 0.011 & 0.000\\
2.31 & 2.26 & 0.501 & 0.001 & 0.005 & 0.005 & 0.039 & 0.045 & 0.004 & 0.000\\
4.07 & 3.97 & 0.479 & 0.000 & 0.001 & 0.001 & 0.037 & 0.026 & 0.001 & 0.000\\
6.01 & 5.86 & 0.489 & 0.000 & 0.000 & 0.000 & 0.027 & 0.027 & 0.000 & 0.000\\
\bottomrule\end{tabular}

\medskip

\begin{tabular}{ccccccccc}\toprule
\multicolumn{9}{l}{(b) median width relative to the uninflated set (clip level $B'$ in parentheses)}\\
$B$ & $\mathrm{CVaR}_\alpha$ & WQ & WQ-Bk & \shortstack{WQ-Bk\\clip} & RS+CP & \shortstack{RS+CP\\clip} & IW-HB & IW-EB\\\midrule
1.00 & 1.00 & 1.082 & 1.071 & 1.071 (1.0) & 1.054 & 1.054 (1.0) & 1.070 & 1.109\\
2.31 & 2.26 & 1.112 & 1.097 & 1.097 (2.3) & 1.072 & 1.071 (2.3) & 1.104 & 1.160\\
4.07 & 3.97 & 1.153 & 1.131 & 1.131 (4.1) & 1.096 & 1.096 (4.1) & 1.134 & 1.229\\
6.01 & 5.86 & 1.195 & 1.164 & 1.164 (6.0) & 1.118 & 1.125 (5.6) & 1.171 & 1.317\\
\bottomrule\end{tabular}

%% file: tables/tab_estimated.tex
\begin{tabular}{llcccccccc}\toprule
estimator & $n_w$ & $\widetilde B$ & $\|\widehat v-v\|_{L^1}$ & $\mathrm{TV}(Q,\widehat Q)$ & $\Delta$ & $\gamma_m(\delta)$ & fail (none) & fail (Bern.) & med.\ cov.\ (Bern.)\\\midrule
logistic & 250 & 6.16 & 0.121 & 0.0579 & 0.0028 & 0.055 & 0.492 & 0.000 & 0.9554\\
logistic & 1000 & 6.47 & 0.117 & 0.0580 & 0.0021 & 0.057 & 0.446 & 0.000 & 0.9577\\
logistic & 4000 & 4.47 & 0.041 & 0.0202 & 0.0047 & 0.045 & 0.568 & 0.000 & 0.9433\\
kde & 1000 & 36.47 & 0.686 & 0.1888 & 0.0032 & 0.196 & 0.522 & 0.000 & 1.0000\\
\bottomrule\end{tabular}

%% file: tables/tab_real.tex
\begin{tabular}{llrrcccccccc}\toprule
\multicolumn{12}{l}{(a) fraction of calibration sets with $P_e>\alpha$}\\
dataset & tilt & $B$ & $\mathrm{CVaR}_\alpha$ & none & WQ & WQ-Bk & \shortstack{WQ-Bk\\clip} & RS+CP & \shortstack{RS+CP\\clip} & IW-HB & IW-EB\\\midrule
Wine & original & 57.6 & 30.8 & 0.532 & 0.000 & 0.000 & 0.000 & 0.024 & 0.004 & 0.000 & 0.000\\
Wine & B=4 & 4.0 & 2.0 & 0.496 & 0.000 & 0.001 & 0.001 & 0.043 & 0.028 & 0.000 & 0.000\\
Abalone & original & 43.9 & 20.6 & 0.513 & 0.000 & 0.000 & 0.000 & 0.027 & 0.003 & 0.000 & 0.000\\
Abalone & B=4 & 4.0 & 2.5 & 0.474 & 0.000 & 0.002 & 0.002 & 0.038 & 0.042 & 0.000 & 0.000\\
Concrete & original & 7.1 & 6.3 & 0.482 & 0.000 & 0.001 & 0.001 & 0.030 & 0.026 & 0.001 & 0.000\\
Concrete & B=4 & 4.0 & 3.5 & 0.512 & 0.000 & 0.003 & 0.003 & 0.042 & 0.029 & 0.002 & 0.000\\
CCPP & original & 13.2 & 6.9 & 0.493 & 0.000 & 0.000 & 0.000 & 0.041 & 0.031 & 0.000 & 0.000\\
CCPP & B=4 & 4.0 & 2.7 & 0.473 & 0.000 & 0.000 & 0.000 & 0.040 & 0.027 & 0.000 & 0.000\\
\bottomrule\end{tabular}

\medskip

\begin{tabular}{llrrccccccc}\toprule
\multicolumn{11}{l}{(b) median width relative to the uninflated set (clip level $B'$ in parentheses)}\\
dataset & tilt & $B$ & $\mathrm{CVaR}_\alpha$ & WQ & WQ-Bk & \shortstack{WQ-Bk\\clip} & RS+CP & \shortstack{RS+CP\\clip} & IW-HB & IW-EB\\\midrule
Wine & original & 57.6 & 30.8 & 1.629 & 1.444 & 1.416 (34.7) & 1.473 & 1.354 (17.7) & 1.753 & $\infty$\\
Wine & B=4 & 4.0 & 2.0 & 1.081 & 1.071 & 1.071 (4.0) & 1.080 & 1.069 (2.7) & 1.117 & 1.124\\
Abalone & original & 43.9 & 20.6 & 2.041 & 1.782 & 1.752 (27.4) & 1.937 & 1.512 (14.6) & 2.330 & $\infty$\\
Abalone & B=4 & 4.0 & 2.5 & 1.167 & 1.147 & 1.147 (4.0) & 1.145 & 1.132 (3.0) & 1.185 & 1.214\\
Concrete & original & 7.1 & 6.3 & 1.936 & 1.496 & 1.496 (7.1) & 1.326 & 1.326 (6.0) & 1.496 & $\infty$\\
Concrete & B=4 & 4.0 & 3.5 & 1.534 & 1.416 & 1.416 (4.0) & 1.408 & 1.322 (3.8) & 1.417 & 2.027\\
CCPP & original & 13.2 & 6.9 & 1.117 & 1.094 & 1.094 (13.2) & 1.105 & 1.075 (7.7) & 1.153 & 1.191\\
CCPP & B=4 & 4.0 & 2.7 & 1.061 & 1.055 & 1.055 (4.0) & 1.048 & 1.045 (3.0) & 1.073 & 1.075\\
\bottomrule\end{tabular}

%% file: tables/tab_frontier.tex
\begin{tabular}{llrr ccccccccccc}\toprule
shift & $\delta$ & $m$ & $\mathrm{CVaR}_\alpha/B$ & WQ & WQ-Bk & RS & RS clip ($B'$) & WQ clip ($B'$) & WQ-Bk clip ($B'$) & CP-C & CP+TV & IW-HB & IW-EB & emp.\ oracle\\\midrule
rare, $\pi=0.01$, $B=21$ & 0.05 & 2000 & 0.25 & 1.228 & 1.192 & 1.368 & 1.118 (1) & 1.145 (4) & 1.135 (5) & 1.991 & \textbf{1.117} & 1.566 & $\infty$ & 1.059\\
rare, $\pi=0.01$, $B=21$ & 0.05 & 5000 & 0.25 & 1.120 & 1.104 & 1.184 & 1.093 (1) & 1.103 (8) & 1.095 (10) & 1.883 & \textbf{1.092} & 1.270 & 1.435$^{2\%}$ & 1.039\\
rare, $\pi=0.01$, $B=21$ & 0.05 & 20000 & 0.25 & 1.053 & \textbf{1.048} & 1.078 & 1.076 (21) & 1.053 (21) & \textbf{1.048} (21) & 1.818 & 1.072 & 1.110 & 1.094 & 1.019\\
rare, $\pi=0.01$, $B=21$ & 0.01 & 2000 & 0.25 & 1.331 & 1.272 & 1.715$^{27\%}$ & \textbf{1.142} (1) & 1.175 (3) & 1.161 (3) & 2.182 & 1.145 & $\infty$ & $\infty$ & 1.076\\
rare, $\pi=0.01$, $B=21$ & 0.01 & 5000 & 0.25 & 1.158 & 1.139 & 1.275 & 1.106 (1) & 1.117 (5) & 1.110 (6) & 1.960 & \textbf{1.106} & 1.367 & $\infty$ & 1.052\\
rare, $\pi=0.01$, $B=21$ & 0.01 & 20000 & 0.25 & 1.068 & \textbf{1.061} & 1.114 & 1.080 (1) & 1.068 (21) & \textbf{1.061} (21) & 1.847 & 1.080 & 1.145 & 1.136 & 1.027\\
\addlinespace
rare, $\pi=0.02$, $B=21$ & 0.05 & 2000 & 0.45 & 1.374 & 1.298 & 1.387 & \textbf{1.200} (1) & 1.230 (4) & 1.216 (5) & 1.995 & 1.200 & 1.574$^{2\%}$ & $\infty$ & 1.065\\
rare, $\pi=0.02$, $B=21$ & 0.05 & 5000 & 0.45 & 1.175 & \textbf{1.151} & 1.180 & 1.174 (1) & 1.171 (13) & \textbf{1.151} (21) & 1.881 & 1.173 & 1.261 & 1.441$^{3\%}$ & 1.043\\
rare, $\pi=0.02$, $B=21$ & 0.05 & 20000 & 0.45 & 1.074 & \textbf{1.066} & 1.076 & 1.078 (21) & 1.074 (21) & \textbf{1.066} (21) & 1.805 & 1.147 & 1.110 & 1.102 & 1.023\\
rare, $\pi=0.02$, $B=21$ & 0.01 & 2000 & 0.45 & 1.715 & 1.464 & 1.701$^{27\%}$ & 1.229 (1) & 1.263 (3) & 1.243 (3) & 2.164 & \textbf{1.228} & $\infty$ & $\infty$ & 1.089\\
rare, $\pi=0.02$, $B=21$ & 0.01 & 5000 & 0.45 & 1.237 & 1.206 & 1.277 & 1.184 (1) & 1.192 (6) & \textbf{1.182} (8) & 1.955 & 1.184 & 1.368 & $\infty$ & 1.055\\
rare, $\pi=0.02$, $B=21$ & 0.01 & 20000 & 0.45 & 1.095 & \textbf{1.086} & 1.114 & 1.111 (21) & 1.095 (21) & \textbf{1.086} (21) & 1.836 & 1.154 & 1.144 & 1.144 & 1.034\\
\addlinespace
rare, $\pi=0.01$, $B=51$ & 0.05 & 2000 & 0.22 & 1.492 & 1.367 & $\infty$ & 1.125 (1) & 1.154 (4) & 1.140 (4) & $\infty$ & \textbf{1.123} & $\infty$ & $\infty$ & 1.062\\
rare, $\pi=0.01$, $B=51$ & 0.05 & 5000 & 0.22 & 1.202 & 1.173 & 1.359 & \textbf{1.096} (1) & 1.111 (6) & 1.106 (7) & 2.179 & 1.097 & 1.573$^{1\%}$ & $\infty$ & 1.042\\
rare, $\pi=0.01$, $B=51$ & 0.05 & 20000 & 0.22 & 1.082 & 1.073 & 1.135 & 1.073 (1) & 1.074 (22) & \textbf{1.071} (37) & 2.045 & 1.073 & 1.196 & 1.196 & 1.024\\
rare, $\pi=0.01$, $B=51$ & 0.01 & 2000 & 0.22 & $\infty$ & 1.637 & $\infty$ & \textbf{1.145} (1) & 1.178 (3) & 1.161 (1) & $\infty$ & 1.147 & $\infty$ & $\infty$ & 1.085\\
rare, $\pi=0.01$, $B=51$ & 0.01 & 5000 & 0.22 & 1.286 & 1.240 & 1.619$^{17\%}$ & \textbf{1.111} (1) & 1.127 (5) & 1.119 (5) & 2.349 & 1.111 & $\infty$ & $\infty$ & 1.053\\
rare, $\pi=0.01$, $B=51$ & 0.01 & 20000 & 0.22 & 1.107 & 1.096 & 1.193 & 1.082 (1) & 1.085 (14) & \textbf{1.081} (16) & 2.101 & 1.082 & 1.257 & 1.332 & 1.029\\
\addlinespace
Gaussian, $c=1$ & 0.05 & 2000 & 0.40 & 1.564 & 1.396 & 1.754$^{41\%}$ & \textbf{1.281} (10) & 1.498 (12) & 1.383 (16) & 1.947 & $\infty$ & 2.007$^{24\%}$ & $\infty$ & 1.112\\
Gaussian, $c=1$ & 0.05 & 5000 & 0.40 & 1.216 & 1.180 & 1.237 & \textbf{1.151} (12) & 1.214 (21) & 1.179 (25) & 1.824 & $\infty$ & 1.353 & $\infty$ & 1.062\\
Gaussian, $c=1$ & 0.05 & 20000 & 0.40 & 1.086 & 1.077 & 1.098 & \textbf{1.073} (16) & 1.086 (33) & 1.077 (33) & 1.694 & $\infty$ & 1.146 & 1.165 & 1.034\\
Gaussian, $c=1$ & 0.01 & 2000 & 0.40 & $\infty$ & 1.731 & $\infty$ & \textbf{1.422} (9) & -- & 1.628 (10) & $\infty$ & $\infty$ & $\infty$ & $\infty$ & 1.171\\
Gaussian, $c=1$ & 0.01 & 5000 & 0.40 & 1.314 & 1.261 & 1.392 & \textbf{1.223} (10) & 1.306 (18) & 1.256 (18) & 1.938 & $\infty$ & 1.558$^{1\%}$ & $\infty$ & 1.107\\
Gaussian, $c=1$ & 0.01 & 20000 & 0.40 & 1.113 & \textbf{1.100} & 1.134 & 1.102 (16) & 1.113 (29) & \textbf{1.100} (33) & 1.725 & $\infty$ & 1.177 & 1.245 & 1.044\\
\bottomrule\end{tabular}

%% file: tables/tab_rerun.tex
\begin{tabular}{llc ccc ccc}\toprule
 & & WQ-Bern.\ & \multicolumn{3}{c}{RS+CP coverage over $300$ reruns} & \multicolumn{3}{c}{RS+CP width / WQ width}\\
\cmidrule(lr){4-6}\cmidrule(lr){7-9}
shift & set & coverage & min & max & IQR & min & median & max\\\midrule
tilt, $c=2$, $B=4$ & 0 & 0.961 & 0.944 & 0.969 & 0.0045 & 0.922 & 0.972 & 1.048\\
tilt, $c=2$, $B=4$ & 1 & 0.962 & 0.951 & 0.964 & 0.0034 & 0.943 & 0.979 & 1.010\\
rare, $\pi=0.01$, $B=21$ & 0 & 0.964 & 0.933 & 0.988 & 0.0128 & 0.838 & 1.039 & 1.266\\
rare, $\pi=0.01$, $B=21$ & 1 & 0.966 & 0.935 & 0.993 & 0.0129 & 0.832 & 1.051 & 1.372\\
Gaussian, $c=1$ & 0 & 0.975 & 0.927 & 0.996 & 0.0167 & 0.756 & 1.018 & 1.406\\
Gaussian, $c=1$ & 1 & 0.980 & 0.941 & 0.995 & 0.0131 & 0.770 & 1.002 & 1.304\\
\bottomrule\end{tabular}

%% file: tables/tab_est.tex
\begin{tabular}{lrr cc ccc c cc}\toprule
 & & & \multicolumn{5}{c}{estimated functionals ($n_s$ source draws, $\delta'=\delta/5$)} & exact & \multicolumn{2}{c}{exact functionals}\\
\cmidrule(lr){4-8}\cmidrule(lr){9-9}\cmidrule(lr){10-11}
shift & $n_s$ & $m$ & WQ & WQ-Bk & WQ clip ($B'$) & WQ-Bk clip ($B'$) & RS clip ($B'$) & RS & WQ-Bk clip ($B'$) & RS clip ($B'$)\\\midrule
tilt, $c=2$, $B=4$ & $10^{4}$ & 5000 & 1.094 & 1.080 & 1.094 (4) & 1.080 (4) & 1.062 (4) & \textbf{1.057} & 1.077 (4) & 1.059 (4)\\
tilt, $c=2$, $B=4$ & $10^{4}$ & 20000 & 1.041 & 1.037 & 1.041 (4) & 1.037 (4) & \textbf{1.028} (4) & 1.029 & 1.035 (4) & 1.026 (4)\\
tilt, $c=2$, $B=4$ & $10^{5}$ & 5000 & 1.087 & 1.075 & 1.087 (4) & 1.075 (4) & 1.055 (4) & \textbf{1.049} & 1.072 (4) & 1.054 (4)\\
tilt, $c=2$, $B=4$ & $10^{5}$ & 20000 & 1.040 & 1.036 & 1.040 (4) & 1.036 (4) & \textbf{1.028} (4) & 1.030 & 1.035 (4) & 1.025 (4)\\
tilt, $c=2$, $B=4$ & $10^{6}$ & 5000 & 1.084 & 1.075 & 1.084 (4) & 1.075 (4) & 1.055 (4) & \textbf{1.047} & 1.071 (4) & 1.048 (4)\\
tilt, $c=2$, $B=4$ & $10^{6}$ & 20000 & 1.040 & 1.035 & 1.040 (4) & 1.035 (4) & 1.027 (4) & \textbf{1.026} & 1.034 (4) & 1.025 (4)\\
\addlinespace
rare, $\pi=0.01$, $B=21$ & $10^{4}$ & 5000 & 1.357 & 1.291 & 1.357 (21) & 1.291 (21) & \textbf{1.190} (21) & 1.193 & 1.098 (10) & 1.094 (1)\\
rare, $\pi=0.01$, $B=21$ & $10^{4}$ & 20000 & 1.124 & 1.110 & 1.124 (21) & 1.110 (21) & 1.087 (21) & \textbf{1.074} & 1.046 (21) & 1.072 (21)\\
rare, $\pi=0.01$, $B=21$ & $10^{5}$ & 5000 & 1.203 & \textbf{1.178} & 1.205 (21) & 1.183 (21) & 1.200 (1) & 1.191 & 1.095 (10) & 1.091 (1)\\
rare, $\pi=0.01$, $B=21$ & $10^{5}$ & 20000 & 1.083 & \textbf{1.074} & 1.083 (21) & \textbf{1.074} (21) & 1.087 (21) & 1.075 & 1.048 (21) & 1.076 (21)\\
rare, $\pi=0.01$, $B=21$ & $10^{6}$ & 5000 & 1.145 & 1.127 & 1.125 (8) & 1.117 (8) & \textbf{1.114} (1) & 1.182 & 1.096 (10) & 1.092 (1)\\
rare, $\pi=0.01$, $B=21$ & $10^{6}$ & 20000 & 1.062 & \textbf{1.055} & 1.062 (21) & \textbf{1.055} (21) & 1.075 (21) & 1.076 & 1.047 (21) & 1.078 (21)\\
\addlinespace
Gaussian, $c=1$ & $10^{4}$ & 5000 & 1.660 & 1.437 & 1.660 (33) & 1.437 (33) & \textbf{1.241} (33) & 1.245 & 1.181 (25) & 1.150 (12)\\
Gaussian, $c=1$ & $10^{4}$ & 20000 & 1.165 & 1.145 & 1.165 (33) & 1.145 (33) & 1.104 (33) & \textbf{1.099} & 1.078 (33) & 1.073 (16)\\
Gaussian, $c=1$ & $10^{5}$ & 5000 & 1.344 & 1.286 & 1.339 (29) & 1.287 (29) & 1.250 (25) & \textbf{1.244} & 1.183 (25) & 1.143 (12)\\
Gaussian, $c=1$ & $10^{5}$ & 20000 & 1.118 & 1.106 & 1.118 (33) & 1.106 (33) & 1.103 (25) & \textbf{1.097} & 1.078 (33) & 1.075 (16)\\
Gaussian, $c=1$ & $10^{6}$ & 5000 & 1.242 & 1.208 & 1.240 (25) & 1.208 (25) & \textbf{1.170} (14) & 1.242 & 1.181 (25) & 1.154 (12)\\
Gaussian, $c=1$ & $10^{6}$ & 20000 & 1.094 & \textbf{1.083} & 1.094 (33) & \textbf{1.083} (33) & 1.091 (18) & 1.093 & 1.076 (33) & 1.073 (16)\\
\bottomrule\end{tabular}